\documentclass[11pt,a4paper]{article}

\usepackage[a4paper,top=1in,bottom=1in,left=1in,right=1in]{geometry} 
\usepackage{setspace}
\usepackage[authoryear]{natbib}  
\usepackage{amsmath,amssymb,amsfonts,amsthm,bm}
\usepackage{graphicx,color,xcolor}
\usepackage{multirow,booktabs,array}
\usepackage{float,subfigure,rotating}
\usepackage{algorithm,algorithmicx,algpseudocode}
\usepackage{url,hyperref}
\usepackage{lipsum}
\usepackage{indentfirst}
\usepackage{authblk}
\usepackage{textcomp}
\usepackage{threeparttable}
\usepackage{pdflscape}
\usepackage{enumerate}
\usepackage{amsmath,amssymb,amsthm}
\usepackage{mathrsfs}
\usepackage{bm}
\usepackage{array}
\usepackage{booktabs}
\usepackage{tabularx}
\usepackage{caption}
\usepackage{geometry}
\usepackage{lastpage}

\newcommand{\trans}{^{\mbox{\tiny{T}}}}
\newtheorem{assumption}{Assumption}  

\def\tr{\mbox{tr}}
\def\r{\mbox{R}}

\newcommand{\I}{\mathbb{I}}
\theoremstyle{remark}
\newtheorem*{Remark}{Remark} 
\theoremstyle{plain}  
\newtheorem{Theorem}{Theorem}
\newtheorem{Corollary}{Corollary}

\newtheorem{lemma}{Lemma}

\newtheorem{example}{Example}
\newtheorem{definition}{Definition}

\date{\today}  

\hypersetup{
	colorlinks=true,
	linkcolor=blue,
	anchorcolor=blue,
	citecolor=blue,
	urlcolor=blue,
	CJKbookmarks=true
}

\begin{document}
	
	\title{Structural Effect and Spectral Enhancement of High-Dimensional Regularized Linear Discriminant Analysis}

	\author{Yonghan Zhang}
	\author{Zhangni Pu}
	\author{Lu Yan}
	\author{Jiang Hu}
	\affil{School of Mathematics $\&$ Statistics, Northeast Normal University, China}
	\date{\today}
	\maketitle
	
	\begin{abstract}
		Regularized linear discriminant analysis (RLDA) is a widely used tool for classification and dimensionality reduction, but its performance in high-dimensional scenarios is inconsistent. Existing theoretical analyses of RLDA often lack clear insight into how data structure affects classification performance. To address this issue, we derive a non-asymptotic approximation of the misclassification rate and thus analyze the structural effect and structural adjustment strategies of RLDA. Based on this, we propose the Spectral Enhanced Discriminant Analysis (SEDA) algorithm, which optimizes the data structure by adjusting the spiked eigenvalues of the population covariance matrix. By developing a new theoretical result on eigenvectors in random matrix theory, we derive an asymptotic approximation on the misclassification rate of SEDA. The bias correction algorithm and parameter selection strategy are then obtained. Experiments on synthetic and real datasets show that SEDA achieves higher classification accuracy and dimensionality reduction compared to existing LDA methods.
	\end{abstract}
	\section{Introduction}
	Linear discriminant analysis (LDA) is a cornerstone of statistical classification, originally introduced in Fisher's seminal work. Its interpretability and effectiveness have led to broad applications in various fields. Specifically, \cite{swets1996Using} used LDA for face image recognition; \cite{pomeroy2002Prediction} and \cite{gurunathan2004Identifying} applied it to gene expression pattern recognition; and \cite{park2003Lower} employed it for dimensionality reduction of text data. With the increase in the dimensionality of modern datasets, the application of LDA on high-dimensional data has received widespread attention. Its appeal lies in the balance between dimensionality reduction and class separation, making it a go-to tool in both theoretical and applied settings.
	
	Despite its wide applicability, the classical formulation of LDA is fundamentally grounded in a low-dimensional asymptotic regime, where the number of features $p$ remains small relative to the sample size $n$. This assumption is often violated in modern high-dimensional datasets, where $p\ge n$ is the norm rather than the exception. In such settings, sample covariance matrix estimates become unstable, and the discriminant directions derived from them lose reliability. \cite{bickel2008Covariance} established that asymptotically, where $p/n\to\infty$, the classification performance of empirical LDA deteriorates to the level of random guessing. \cite{shao2011Sparse} subsequently confirmed that ensuring consistency for empirical LDA requires $p/n\to 0$.
	
	In high-dimensional scenarios, regularization techniques are commonly used to optimize the estimation of the covariance matrix. For example, \cite{chen2011Regularized} investigated the regularized Hotelling's $T^2$ test, while \cite{ledoit2004Honey} examined regularized estimation for Markowitz portfolios. This approach has also been widely used for other high-dimensional statistical problems, including works by \cite{cai2011Direct}, \cite{buhlmann2013Statistical}, and \cite{wang2016High}. Within discriminant analysis, \cite{friedman1989Regularized} and \cite{guo2007Regularized} introduced and developed regularized linear discriminant analysis (RLDA). A series of studies based on random matrix theory have subsequently emerged, including \cite{zollanvari2015Generalized}'s study of RLDA misclassification rates, \cite{dobriban2018Highdimensional}'s study of dimensional effects under the random effects assumption, and \cite{wang2018Dimension}'s study under certain structural assumptions. Specifically, consider two classes $C_1:\bm{x}\sim N(\bm{\mu}_1,\bm{\Sigma})$ and $C_2: \bm{x}\sim N(\bm{\mu}_2,\bm{\Sigma})$ having equal prior probabilities. When $\bm{\mu}_1$, $\bm{\mu}_2$, and $\bm{\Sigma}$ are known, the Bayes' classification rule is
	
	\begin{align}
		D(\bm{x})=\I \left\{\left(\bm{x}-\frac{\bm{\mu}_1+\bm{\mu}_2}{2}\right) \trans \bm{\Sigma}^{-1}\left(\bm{\mu}_1-\bm{\mu}_2\right)>0\right\},
	\end{align}
	which classifies $\bm{x}$ into $C_1$ when $D(\bm{x})=1$. $\I(\cdot)$ is the indicator function, and the true Bayes error rate is
	\begin{align}\label{2}
		\r(\bm{x})=&\frac{1}{2} \Pr \left\{ D(\bm{x})=0|\bm{x} \sim N(\bm{\mu}_1,\bm{\Sigma})\right\}+\frac{1}{2} \Pr \left\{ D(\bm{x})=1|\bm{x} \sim N(\bm{\mu}_2,\bm{\Sigma})\right\}\nonumber\\
		=&\Phi\left(- \frac{1}{2}\sqrt{\left(\bm{\mu}_1-\bm{\mu}_2\right) \trans \bm{\Sigma}^{-1}\left(\bm{\mu}_1-\bm{\mu}_2\right)}\right), 
	\end{align}
	where $\Phi(\cdot)$ is the standard normal distribution function.
	
	Let $\{\bm{x}_{1,j},j=1,\dots,n_1\}$ and $\{\bm{x}_{2,j},j=1,\dots,n_2\}$ be random samples drawn independently from $N(\bm{\mu}_1,\bm{\Sigma})$ and $N(\bm{\mu}_2,\bm{\Sigma})$, respectively. We can estimate $\bm{\mu}_1,\bm{\mu}_2$ and $\bm{\Sigma}$ by the sample means and the pooled sample covariance matrix,
	\begin{align*}
		&\bar{\bm{x}}_i=\frac{1}{n_i}\sum_{j=1}^{n_i} \bm{x}_{i,j},~i=1,2,~\bm{S}_n=\frac{1}{n-2}\sum_{i=1}^2 \sum_{j=1}^{n_i}(\bm{x}_{i,j}-\bar{\bm{x}}_i)(\bm{x}_{i,j}-\bar{\bm{x}}_i) \trans,
	\end{align*}
	where $n=n_1+n_2$. For a given $\lambda>0$, substituting the estimation into Bayes' rule yields the RLDA classifier as follows:
	\begin{align}
		D_{\text{RLDA}}(\bm{x})=\I\left\{\left(\bm{x}-\frac{\bar{\bm{x}}_1+\bar{\bm{x}}_2}{2}\right)\trans \left(\bm{S}_n+\lambda \bm{I}_p\right)^{-1}\left(\bar{\bm{x}}_1-\bar{\bm{x}}_2\right)>0\right\}.
	\end{align}
	After some simple calculations, the misclassification rate of RLDA is
	\begin{align}
		\r_{\text{RLDA}}(\lambda)=\frac{1}{2}\sum_{i=1}^2 \Phi\left(\frac{(-1)^i\left(2\bm{\mu}_i-\bar{\bm{x}}_1-\bar{\bm{x}}_2\right)\trans \left(\bm{S}_n+\lambda \bm{I}_p\right)^{-1}\left(\bar{\bm{x}}_1-\bar{\bm{x}}_2\right) }{2\sqrt{\left(\bar{\bm{x}}_1-\bar{\bm{x}}_2\right)\trans \left(\bm{S}_n+\lambda \bm{I}_p\right)^{-1}\bm{\Sigma} \left(\bm{S}_n+\lambda \bm{I}_p\right)^{-1} \left(\bar{\bm{x}}_1-\bar{\bm{x}}_2\right)}}\right).
	\end{align}

	The performance of RLDA in high-dimensional scenarios has always been a topic of interest. \cite{dobriban2018Highdimensional} and \cite{wang2018Dimension} derived the asymptotic approximation of the misclassification rate under different assumptions, revealing the impact of the ratio $p/n$ and the regularization parameter $\lambda$. Among them, \cite{dobriban2018Highdimensional} assumed that $\bm{\mu}_1$ and $\bm{\mu}_2$ are random, while \cite{wang2018Dimension} replaced it with structural assumptions. These results all suggest that the misclassification rate is influenced by the structures of $\bm{\mu}_1$, $\bm{\mu}_2$, and $\bm{\Sigma}$, but this influence has not been well explained. Recent articles have introduced different methods to improve the performance of RLDA in high-dimensional scenarios. \cite{li2025Spectrallycorrected} proposed spectrally-corrected and regularized LDA (SRLDA), which improves accuracy by constructing an estimation of $\bm{\Sigma}$ under the spiked model. The scale invariant linear discriminant analysis (SIDA) proposed by \cite{li2025Highdimensional} is equivalent to classifying the standardized data. These methods essentially reduce classification error by adjusting the structure, but they have significant limitations in their application scenarios. \cite{li2025Spectrallycorrected} required all eigenvalues of $\bm{\Sigma}$ to be equal except for a finite number of spiked eigenvalues. \cite{li2025Highdimensional} can only achieve effective adjustment when the data correlation is weak. In this paper, we establish a non-asymptotic approximation of the misclassification rate, further discuss the impact of the structure, and propose the Spectral Enhancement Discriminant Analysis (SEDA) classifier that adjusts the structure in more general scenarios. Following is a summary of the contributions of our work:
	\begin{itemize}
		\item We derive a closed-form expression for the misclassification rate of RLDA under general conditions. \cite{dobriban2018Highdimensional} provided an asymptotic result for this problem under random effect. \cite{wang2018Dimension} relaxed this condition by replacing it with structural assumptions. Based on the latest results from random matrix theory, we propose a non-asymptotic approximation of the misclassification rate without these technical assumptions, further explaining the impact of the structure. The new results indicate that the eigenvectors corresponding to small eigenvalues may play a more important role in the classification process. This provides a new strategy for improving RLDA. 
		\item We propose a novel SEDA classifier that strategically adjusts the structure of spiked eigenvalues to enhance classification performance. In contrast to \cite{li2025Spectrallycorrected}, our method does not impose the restrictive assumption of equal non-spiked eigenvalues, thereby significantly broadening its applicability. Crucially, these structural refinements preserve the full utilization of sample information without incurring additional loss. Building on a new theoretical advancement concerning eigenvectors in random matrix theory, we derive an accurate asymptotic approximation of the SEDA misclassification rate. Furthermore, we develop a principled approach to obtain theoretically optimal parameters. To accommodate settings with unequal sample sizes, we introduce a bias-corrected variant of the SEDA classifier. Finally, we extend the method to the multi-class setting by proposing a tailored dimensionality reduction algorithm based on the SEDA framework. Notably, we derive the limit of the inner product of the spiked eigenvectors of the sample covariance matrix with any unit vector under the generalized spiked model, which is a new theoretical result.
	\end{itemize}
	\paragraph{Paper Organization.} The remainder of this paper is organized as follows: In Section \ref{sec2}, we provide a non-asymptotic approximation of the misclassification rate for RLDA and discuss its structural effects. In Section \ref{sec3}, we propose the SEDA classifier and give an asymptotic approximation for its misclassification rate in Subsection \ref{sub3.1}. In Subsection \ref{sub3.2}, we offer a bias-corrected SEDA for handling imbalanced data. In Subsection \ref{sub3.3}, we propose a method for solving the theoretically optimal parameters of SEDA. In Section \ref{sec4}, we validate the effectiveness of the theories and compare the performance of SEDA with existing methods through numerical simulations. In Section \ref{sec5}, we extend SEDA to the multi-class classification problem and examine the algorithm's classification and dimensionality reduction performance using real datasets. The final section analyzes the conclusion. All technical details are relegated to the Appendix. 
	
	\paragraph{Notation.}The notation $\|\cdot\|$ is the Euclidean norm for vectors and the operator norm for matrices. The notation $\propto$ defines 'is proportional to'. For a symmetric matrix $\bm{A}\in\mathbb{R}^{p\times p}$, $\lambda_{min}(\bm{A})$ and $\lambda_{max}(\bm{A})$ denote the minimum and maximum eigenvalues of $\bm{A}$, respectively. $s_j$ and $\bm{v}_j$ are the $j$-th largest eigenvalue and corresponding eigenvector of $\bm{\Sigma}$, respectively. $a_j$ and $\bm{u}_j$ are the $j$-th largest eigenvalue and corresponding eigenvector of $\bm{S}_n$, respectively. All vectors in the article are column vectors.
	
	\section{Structural effect of RLDA}\label{sec2}
	The approximation of $\r_{\text{RLDA}}(\lambda)$ is determined by the structure between $\bm{\mu}_1$, $\bm{\mu}_2$, and $\bm{\Sigma}$. Denote $\bm{\mu}=\bm{\Sigma}^{-\frac{1}{2}}\left(\bm{\mu}_1-\bm{\mu}_2\right)$ and $\bm{\Sigma}=\sum_{i=1}^{p}s_{i}\bm{v}_i\bm{v}_i\trans$, where $s_1\geq s_2\geq\dots\geq s_p\geq 0$ are the ordered eigenvalues. The structure can be characterized by two components: the eigenvalue spectrum $(s_{1}, \dots, s_{p})$, and the projection coefficients of $\bm{\mu}$ onto the basis of eigenvectors $(\langle \bm{v}_1, \bm{\mu}\rangle, \ldots,\langle \bm{v}_p, \bm{\mu}\rangle)$. These components can be formally represented using the following two probability measures:
	\begin{align}\label{Hn}
		H_{n}(s):=\frac{1}{p} \sum_{i=1}^{p} \I\{s \geq s_{i}\}, \quad G_{n}(s):=\frac{1}{\|\bm{\mu}\|^{2}} \sum_{i=1}^{p}\left\langle\bm{\mu}, \bm{v}_i\right\rangle^{2} \I\{s \geq s_{i}\},
	\end{align}
	
	We then introduce the Random Matrix Theory (RMT) tools we will be using. An important mathematical tool in RMT is the \emph{Mar{\v{c}}enko-Pastur equation}. For each $\lambda,y>0$ and any distribution $F$, define $m(-\lambda):=m(-\lambda;F,y)$ as the unique solution of the \emph{Mar{\v{c}}enko-Pastur equation} \citep{marcenko1967Distribution,elkaroui2008Spectrum}
	\begin{align}
		m(-\lambda)=\int \frac{1}{s\left[1-y+y \lambda m(-\lambda)\right]+\lambda} d F(s),
	\end{align}
	under the condition $1 - y + y\lambda m(-\lambda) \geq 0$ \citep{wang2015Shrinkage}. 
	
	Further define $m_1(-\lambda):=m_1(-\lambda;F,y)$ as
	\begin{align}
		m_1(-\lambda;F,y)=\frac{\int \frac{s^{2}\left[1-y+y\lambda m(-\lambda)\right]}{\left[s\left[1-y+y\lambda m(-\lambda)\right]+\lambda\right]^{2}} d F(s)}{1+y\int \frac{\lambda s}{\left[s\left[1-y+y \lambda m(-\lambda)\right]+\lambda\right]^{2}} d F(s)} .
	\end{align}
	We can construct the following expressions for estimation:
	\begin{align*} 
		& T_{1}(\lambda;H_n,y)=\int \frac{s}{s\left[1-y+y \lambda m(-\lambda;H_n,y)\right]+\lambda} d H_{n}(s),\\
		& U_{1}(\lambda;H_n,G_n,y)=\|\bm{\mu}\|^{2}\int \frac{s}{s\left[1-y+y \lambda m(-\lambda;H_n,y)\right]+\lambda} d G_{n}(s)  ,\\
		& T_{2}(\lambda;H_n,y)=\left[1+ym_1(-\lambda;H_n,y)\right]\int \frac{s^2}{\left[s\left[1-y+y \lambda m(-\lambda;H_n,y)\right]+\lambda\right]^{2}} d H_{n}(s),\\
		& U_{2}(\lambda;H_n,G_n,y)=\|\bm{\mu}\|^{2}\left[1+ym_1(-\lambda;H_n,y)\right]\int \frac{s^2}{\left[s\left[1-y+y \lambda m(-\lambda;H_n,y)\right]+\lambda\right]^{2}} d G_{n}(s) .
	\end{align*}
	
	Within the theoretical framework, the following assumptions are introduced. The main results are established uniformly with respect to the (large) constant $M$ that appears therein.
	\begin{assumption}\label{as11}
		The population covariance matrix $\bm{\Sigma}\in\mathbb{R}^{p\times p}$ is deterministic and satisfies $s_1=\|\bm{\Sigma}\| \leq M$ and $\int s^{-1} d H_{n}(s)\leq M$.
	\end{assumption}	
	\begin{assumption}\label{as12}
		$\bm{\mu}$ has a bounded Euclidean norm, that is $1 / M \leq \|\bm{\mu}\| \leq M$.
	\end{assumption}
	\begin{assumption}\label{as13}
		$ |1-p/n| \geq 1 / M$, $1 / M \leq p/n_i \leq M, i=1,2$.
	\end{assumption}
	\begin{Remark}
		Assumption \ref{as11} requires the eigenvalues of $\bm{\Sigma}$ to be bounded and not to accumulate near 0. Assumption \ref{as12} requires that $\|\bm{\mu}\|$ be bounded, which helps us characterize the approximation accuracy. Since our statements are non-asymptotic, we do not assume that $p/n$ converges to a value. However, assumption \ref{as13} requires that $p/n_i$ be bounded and bounded $p/n$ away from 1.
	\end{Remark}
	
	After the above discussion, our deterministic approximation of the misclassification rate is shown in the following theorem.
	\begin{Theorem}[Deterministic approximation of RLDA misclassification rate]\label{thm1}
		Under the Assumptions \ref{as11}--\ref{as13}, let $y_{1n}=p/n_1, y_{2n}=p/n_2$ and $y_n=p/n$, for any $1/M \leq \lambda \leq M$, $D>0$ (arbitrarily large) and $\varepsilon>0$ (arbitrarily small), there exists $C=C(M,D)$ such that, with probability at least $1-Cn^{-D}$, the following holds:
		\begin{equation*}
			\left|\r_{\text{RLDA}}(\lambda)-\frac{1}{2}\sum_{i=1}^2\Phi\left(-\frac{U_{1}(\lambda;H_n,G_n,y_n)+(-1)^i\left(y_{1n}-y_{2n}\right)T_{1}(\lambda;H_n,y_n)}{2 \sqrt{U_{2}(\lambda;H_n,G_n,y_n)+\left(y_{1n}+y_{2n}\right) T_{2}(\lambda;H_n,y_n)}}\right)\right| \leq \frac{C}{n^{\left(1-\varepsilon\right)/2}} .
		\end{equation*}
	\end{Theorem}
	\begin{Remark}
		Theorem \ref{thm1} establishes a deterministic approximation for the misclassiﬁcation rate, which is valid at finite $n$ and $p$, and the error bound is uniform (i.e., depends only on the constant $M$). This will contrast with the asymptotic setting in \cite{dobriban2018Highdimensional} and \cite{wang2018Dimension}. Both of these obtained asymptotic approximation for the misclassification rate under the assumption that $n, p\rightarrow\infty$. To make sense of the asymptotic approximation, they both assume that the empirical spectral distribution (ESD) of $\bm{\Sigma}$ converges to a nonrandom distribution function as $p\rightarrow \infty$. Moreover, \cite{dobriban2018Highdimensional} assumes that $\bm{\mu}$ is random, while \cite{wang2018Dimension} assumes that for any $t>0$, as $p\rightarrow \infty$, $\|\bm{\mu}\|^{-1}\bm{\mu}\trans\left(\bm{I}_p+t\bm{\Sigma}^{-1}\right)^{-i}\bm{\mu}\rightarrow h_i(t), i=1,2$. Specific expressions for $h_i(t)$ can be obtained when $\bm{\mu}$ and $\bm{\Sigma}$ have certain special structures. Our results do not require these additional assumptions.
	\end{Remark}
	In particular, if two probability measures $H_n$ and $G_n$ converge weakly to $H$ and $G$ on $[0,\infty)$, respectively. Then we obtain the following asymptotic result immediately from Theorem \ref{thm1} by taking $n,p\rightarrow \infty$ (using \emph{Borel-Cantelli Lemma} to obtain almost sure convergence).
	\begin{Corollary}[Asymptotic misclassification rate for RLDA]\label{cor}
		Under the Assumptions \ref{as11}--\ref{as13}. Further assume $n,p\rightarrow\infty, p/n_1\rightarrow y_1, p/n_2\rightarrow y_2 , p/n\rightarrow y, H_n\Rightarrow H,G_n\Rightarrow G$. Then, almost surely
		\begin{align*}
			\r_{\text{RLDA}}(\lambda)\rightarrow \frac{1}{2}\sum_{i=1}^2\Phi\left(-\frac{U_{1}(\lambda;H,G,y)+(-1)^i\left(y_1-y_2\right)T_{1}(\lambda;H,y)}{2 \sqrt{U_{2}(\lambda;H,G,y)+\left(y_1+y_2\right) T_{2}(\lambda;H,y)}}\right).
		\end{align*}
		Here denotes $\Rightarrow$ as weak convergence.
	\end{Corollary}
	\begin{Remark}
		In contrast to the technical assumption made by \cite{wang2018Dimension}, we only require weak convergence to ensure that the asymptotics are meaningful. Not only that, we have relaxed the bound on the eigenvalues of $\bm{\Sigma}$, and the result was extended to almost surely. Furthermore, our expression more clearly demonstrates the impact of $\bm{\mu}$ and $\bm{\Sigma}$ on classification performance.
	\end{Remark}
	
	From Theorem \ref{thm1}, it can be seen that the contribution of the eigenvector $\bm{v}_j$ depends on the weight $\left\langle\bm{\mu}_1-\bm{\mu}_2,\bm{v}_j\right\rangle^2/s_j$. A counterintuitive conclusion is that the eigenvectors corresponding to small eigenvalues seem to play a more important role in classification tasks. To further discuss the impact of structure on the misclassification rate of RLDA, consider the following examples. Without loss of generality, assume $n_1=n_2$.
	\begin{example} Consider a sparse case where
		\begin{align*}
			H_{n}(s)=\frac{1}{p} \sum_{i=1}^{p} \I\{s \geq s_{i}\}, \quad G_{n}(s)=\I\{s \geq s_{k}\},
		\end{align*}
		with some $k\in\{1,\dots,p\}$.
	\end{example}
	In this example, the Bayes' discriminant direction $\bm{\Sigma}^{-1}(\bm{\mu}_1-\bm{\mu}_2)$ is parallel to the eigenvector $\bm{v}_k$. By direct calculation, it can be verified that $U_{1}^2(\lambda;H_n,G_n,y_n)/[U_{2}(\lambda;H_n,G_n,y_n)+4y_n T_{2}(\lambda;H_n,y_n)]$ is an increasing function of $s_k$. This means that when $\bm{\mu}$ is parallel to the eigenvectors corresponding to the small eigenvalues, the performance of RLDA will deteriorate. This is in contrast to the result of LDA, where LDA's performance is only related to $y_{1n}, y_{2n}$ and $\|\bm{\mu}\|$.
	
	In practice, RLDA also exhibits unstable performance in sparse cases. A natural thought is, can the performance be improved by enhancing the small eigenvalues of $\bm{\Sigma}$? We illustrate this point with an example below.
	\begin{example}\label{exp2}Consider a more general case, for some fixed $d$,
		\begin{align*}
			H_{n}(s)=\frac{1}{p} \sum_{i=1}^{p} \I\{s \geq s_{i}\}, \quad G_{n}(s)=\frac{1}{d} \sum_{i=1}^{d}\left\langle\bm{\mu}, \bm{v}_{k_i}\right\rangle^{2} \I\{s \geq s_{k_i}\},
		\end{align*}
		where $\{k_1,\dots,k_d\}\subset \{1,\dots,p\}$ and $s_{k_1}\geq s_{k_2}\geq \dots\geq s_{k_d}$. Define $H_{g_n}(s)=\frac{1}{p} \sum_{i=1}^{p} \I\{s \geq g(s_{i})\}$ and $G_{g_n}(s)=\frac{1}{d} \sum_{i=1}^{d}\left\langle\bm{\mu}, \bm{v}_{k_i}\right\rangle^{2} \I\{s \geq g(s_{k_i})\}$, with $g(s)=\max\{s,s_{k_1}\}$. This means performing a linear transformation on $\bm{x}$ to amplify the small eigenvalues $s_{k_2}\dots s_{k_d}$ to $s_{k_1}$. Under the conditions of Corollary \ref{cor}, further assume $H_{g_n}\Rightarrow H_g, G_{g_n}\Rightarrow G_g$, it can be verified that
		\begin{align*}
			\frac{U_{1}^2(\lambda;H,G,y)}{U_{2}(\lambda;H,G,y)+4y T_{2}(\lambda;H,y)}\leq \frac{U_{1}^2(\lambda;H_g,G_g,y)}{U_{2}(\lambda;H_g,G_g,y)+4y T_{2}(\lambda;H_g,y)}
		\end{align*}
		the equality holds if and only if $s_{k_1}= \dots= s_{k_d}$. 
	\end{example}
	Example \ref{exp2} illustrates that the performance of RLDA can be improved by amplifying the small eigenvalues of $\bm{\Sigma}$. To more intuitively understand these two examples, we consider a common model: $\bm{\Sigma}=(\rho^{|i-j|})_{p\times p}$ with $|\rho|<1$, which is used for LDA in \cite{bickel2004Theory}. By the \emph{Szeg\H{o} theorem}, we have
	\begin{align*}
		s_k \approx \frac{1-\rho^{2}}{1+\rho^{2}-2 \rho \cos \frac{k \pi}{p+1}}
	\end{align*}
	Thus, $s_1\rightarrow(1+\rho)/(1-\rho)$, $s_{p/2}\rightarrow(1-\rho^2)/(1+\rho^2)$ and $s_p\rightarrow(1-\rho)/(1+\rho)$. The left side of figure \ref{fig2.1} presents the empirical values of the misclassification rate for $\bm{\mu}\propto \bm{v}_k$, $k\in\{1,p/2,p\}$, while the right side presents the results after amplifying $s_p$ by a factor of 20. Figure \ref{fig2.1} visually demonstrates that when $\bm{\mu}$ is parallel to the eigenvectors corresponding to small eigenvalues, the performance of RLDA deteriorates, which can be improved by amplifying the small eigenvalues. These phenomena coincide with the examples we discussed.
	\begin{figure}
		\centering
		\begin{minipage}{0.48\textwidth}
			\includegraphics[width=\textwidth]{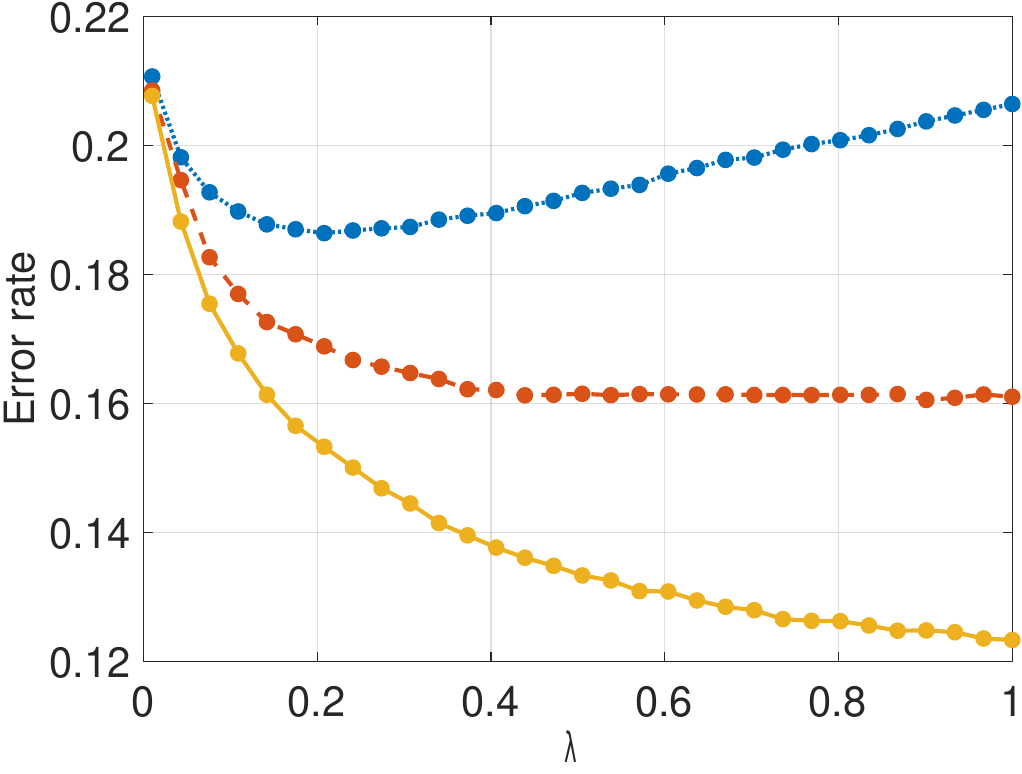}
		\end{minipage}\quad
		\begin{minipage}{0.48\textwidth}
			\includegraphics[width=\textwidth]{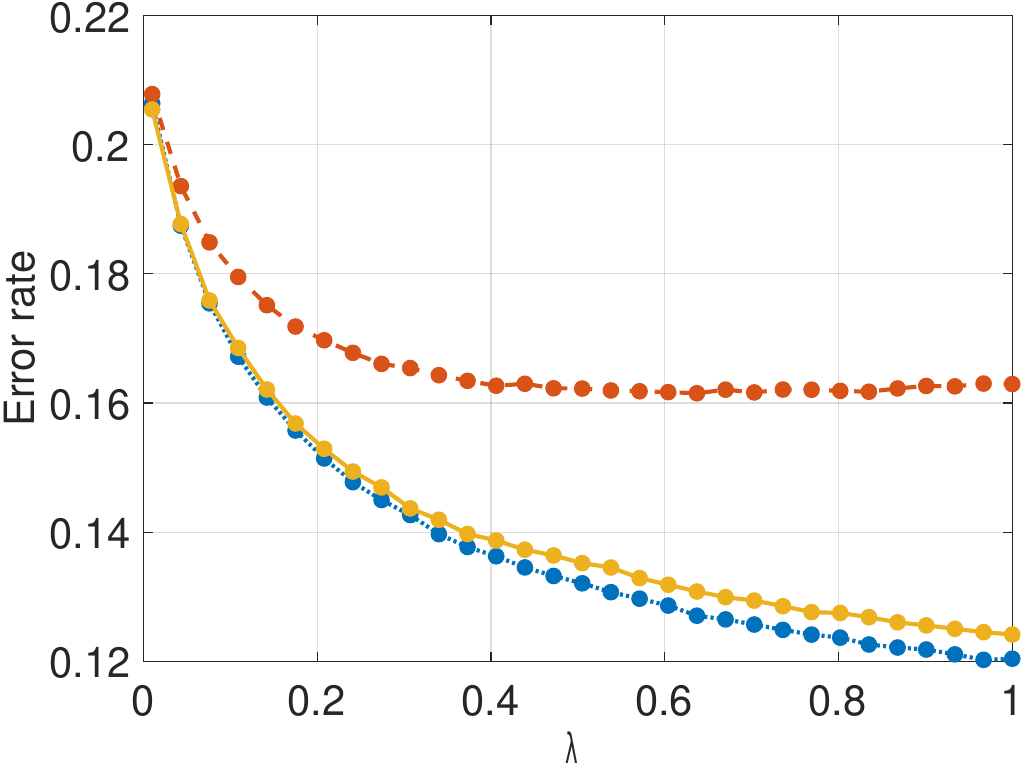}
		\end{minipage}
		\caption{The empirical misclassification rates for $\bm{\Sigma}=(0.5^{|i-j|})_{100\times100}$ and $n_1=n_2=100$. The line stands for $\bm{\mu}\propto \bm{v}_1$; the dashed line is the results for $\bm{\mu}\propto \bm{v}_{50}$ and the dotted line is the one for $\bm{\mu}\propto \bm{v}_{100}$. For all the cases, the true Bayes error rate defined in \eqref{2} is fixed at 10\%.}
		\label{fig2.1}
	\end{figure}
	\section{Spectral enhancement discriminant analysis}\label{sec3}
	In this section, we consider a special scenario: the population covariance matrix possesses a finite number of spiked (outlier) eigenvalues. This spiked model, first proposed by \cite{johnstone2001Distribution}, posits that the bulk of the eigenvalues cluster together, while a small number of "spikes" lie distinctly outside this bulk cluster—either much larger or much smaller. From the discussion in the previous section, it is known that the performance of RLDA suffers from significant instability when the projection of the population mean vector onto the spiked eigenvectors has large magnitude. The following work is dedicated to solving this problem.
	\begin{assumption}[Spiked model]\label{as2}
		Let $p/n\rightarrow y\in(0,1)\cup(1,\infty)$ and $H_n\Rightarrow H$, for any $j\in\mathbb{J}$, $s_j$ satisfys
		\begin{align*}
			\int\frac{s^2dH(s)}{(s_j-s)^2}<\frac{1}{y},
		\end{align*}
		where $\mathbb{J}=\mathbb{J}_1\cup\mathbb{J}_2$, $\mathbb{J}_1=\{1,\dots,r_1\}$, $\mathbb{J}_2=\{p-r_2+1,\dots,p\}$, with fixed $r=r_1+r_2$.
	\end{assumption}
	
	The above model is the so-called generalized spiked model, where $r_1$ and $r_2$ denote the numbers of large and small spiked eigenvalues, respectively. Spiked model encountered in many real applications, such as detection \citep{zhao1986Detection}, EEG signals \citep{davidson2009Functional}, and financial econometrics \citep{kritchman2008Determining,passemier2017Estimation}. Under the framework of high-dimensional random matrix theory, the asymptotic limit of spiked eigenvalues and eigenvectors has been widely and deeply studied \citep{mestre2008Asymptotic,bai2012Estimation,bao2022Statistical,liu2025Asymptotic}. For the sake of simplicity, we assume that $r_1$ and $r_2$ are perfectly known. In our simulations and experiments, we have used the method of \cite{jiang2023Universal} to estimate them.
	
	Under this model assumption, we propose a structural adjustment method called Spectral Enhancement Discriminant Analysis to improve classification performance. For given $\lambda>0$ and $\begin{cases}
		\ell_j\leq0,&j\in\mathbb{J}_1\\0\leq \ell_j<1,&j\in\mathbb{J}_2
	\end{cases}$, the SEDA classifier is given as follows:
	\begin{align}\label{L}
		D_{\text{SEDA}}(\bm{x})=\I\left\{\left(\bm{x}-\frac{\bar{\bm{x}}_1+\bar{\bm{x}}_2}{2}\right)\trans \left(\bm{S}_n+\lambda \bm{\mathcal{I}}\right)^{-1}\left(\bar{\bm{x}}_1-\bar{\bm{x}}_2\right)>0\right\},
	\end{align}
	where $\bm{\bm{\mathcal{I}}}=\bm{I}_p-\sum_{j\in\mathbb{J}}\ell_j\bm{u}_j\bm{u}_j\trans$. Define $\bm{\theta}=(\lambda,\ell_1,\dots,\ell_{r_1},\ell_{p-r_2+1},\dots,\ell_p)$, we can get the misclassification rate of SEDA 
	\begin{align*}
		\r_{\text{SEDA}}(\bm{\theta})=\frac{1}{2}\sum_{i=1}^2 \Phi\left(\frac{(-1)^i\left(2\bm{\mu}_i-\bar{\bm{x}}_1-\bar{\bm{x}}_2\right)\trans \left(\bm{S}_n+\lambda \bm{\mathcal{I}}\right)^{-1}\left(\bar{\bm{x}}_1-\bar{\bm{x}}_2\right) }{2\sqrt{\left(\bar{\bm{x}}_1-\bar{\bm{x}}_2\right)\trans \left(\bm{S}_n+\lambda \bm{\mathcal{I}}\right)^{-1}\bm{\Sigma} \left(\bm{S}_n+\lambda \bm{\mathcal{I}}\right)^{-1} \left(\bar{\bm{x}}_1-\bar{\bm{x}}_2\right)}}\right).
	\end{align*}
	
	The essence of SEDA is to find an appropriate transformation $\bm{X}\mapsto \bm{W}\bm{X}$ to adjust the structure of the covariance matrix $\bm{\Sigma}$. Formally, the transformation enhances small spiked eigenvalues and diminishes large spiked eigenvalues, maintaining the original eigenvectors. If $\ell_j,j\in\mathbb{J}$ are set to zero, SEDA will simplify to RLDA. 
	\subsection{Asymptotic misclassification rate}\label{sub3.1}
	To further investigate the asymptotic misclassification rate of SEDA, we make the following assumptions: our results will be uniform with respect to the positive constant $c$ appearing in this assumptions.
	\begin{assumption}\label{as31}
		~$p,n_1,n_2\rightarrow\infty$, $p/n_i\rightarrow y_i\in(0,\infty), i=1,2$.
	\end{assumption} 
	\begin{assumption}\label{as32}
		The spectral norm of $\bm{\Sigma}$ and the Euclidean norm of $\bm{\mu}$ are bounded, i.e., $1/c\leq\|\bm{\mu}\|\leq c$ and $1/c\leq\|\bm{\Sigma}\|\leq c$.
	\end{assumption}
	\begin{assumption}\label{as33}
		For any $j,k\in\mathbb{J}$, there exists some constant $c>0$ independent of $p$ and $n$, such that
		\begin{align*}
			\min_{j \neq k}\left|\frac{s_k}{s_j}-1\right|>c.
		\end{align*}
	\end{assumption}
	\begin{assumption}\label{as34}
		For given $\begin{cases}
			\ell_j\leq 0,&j\in\mathbb{J}_1\\0\leq \ell_j<1,&j\in\mathbb{J}_2
		\end{cases}$, Let $H_{f_n}(s)=\frac{1}{p}\sum_{i=1}^{p}\I\{s\geq f(s_i)\}\Rightarrow H_f(s)$ and $G_{f_n}(s)=\frac{1}{\|\bm{\mu}\|^2}\sum_{i=1}^{p}\left\langle\bm{\mu}, \bm{v}_i\right\rangle^{2} \I\{s \geq f(s_{i})\}\Rightarrow G_f(s)$, where  
		\begin{align*}
			f(s_i)=\left[1+\sum_{j\in\mathbb{J}}(\ell_j/1-\ell_j)\chi_j(i)\right]s_i,
		\end{align*}
		and $\{\chi_j(i)\}$ is defined by
		\begin{align*}
			\chi_j(i)=\left\{
			\begin{array}{ll}
				1-\sum_{k=1,k\neq j}^{p}\left(\frac{s_j}{s_k-s_j}-\frac{\omega_j}{s_k-\omega_j}\right), & j=i\\
				\frac{s_j}{s_i-s_j}-\frac{\omega_j}{s_i-\omega_j}, & j\neq i
			\end{array}
			\right.
		\end{align*}
		$\{\omega_j\}$ are the solutions to the following equation in $\omega$ with a descending order,
		\begin{align}\label{c7}
			\frac{1}{p}\sum_{i=1}^{p}\frac{s_i}{s_i-\omega}=\frac{1}{y}.
		\end{align}
	\end{assumption}
	\begin{Remark}
		Assumptions \ref{as31} and \ref{as32} are similar to Assumption \ref{as11}--\ref{as13}, while are two common conditions in random matrix theory. Assumption \ref{as33} ensures the gaps of adjacent spiked eigenvalues have a constant lower bound. Assumption \ref{as34} requires the ESD of the transformed covariance matrix to converge. It is easy to see that $\chi_j(i)\rightarrow 0$ when $j\neq i$; therefore, $f(\cdot)$ actually amplifies small spiked eigenvalues and diminishes large spiked eigenvalues without changing their eigenvectors.
	\end{Remark}
	Based on the above assumptions, we can establish an asymptotic approximation of the misclassification rate of SEDA. Before this, we present a key lemma in the proof of the main theorem.
	\begin{lemma}[Convergence of sample spiked eigenvectors]\label{lem2.4}
		Under the Assumptions \ref{as2}--\ref{as34}, for any $j\in\mathbb{J}$ and any deterministic unit vectors $\bm{\xi}\in \mathbb{R}^p$, we have
		\begin{align}
			\left|\bm{\xi}\trans \bm{u}_j\bm{u}_j\trans \bm{\xi}-\sum_{i=1}^{p}\chi_j(i)\bm{\xi}\trans \bm{v}_i\bm{v}_i\trans \bm{\xi}\right|\xrightarrow{a.s.} 0,
		\end{align}
	\end{lemma}
	\begin{Remark}
		Lemma \ref{lem2.4} extends the limiting result for the angle between the true and estimated spiked eigenvectors \citep{li2025Spectrallycorrected}. We relax the assumption that non-spiked eigenvalues are equal and generalize the result to the generalized spiked model. This enables further exploration of the theoretical properties of the SEDA classifier.
	\end{Remark}
	Then, we obtain the asymptotic misclassification rate of SEDA as shown in the following theorem.
	\begin{Theorem}[Asymptotic misclassification rate for SEDA]\label{thm2}
		Under the Assumptions \ref{as2}--\ref{as34}, for any $\lambda>0$ and $\begin{cases}
			\ell_j\leq 0,&j\in\mathbb{J}_1\\0\leq \ell_j<1,&j\in\mathbb{J}_2
		\end{cases}$, almost surely
		\begin{equation*}
			\r_{\text{SEDA}}(\bm{\theta})\rightarrow\frac{1}{2}\sum_{i=1}^2\Phi\left(-\frac{U_{1}(\lambda;H_f,G_f,y)+(-1)^i\left(y_1-y_2\right)T_{1}(\lambda;H_f,y)}{2 \sqrt{U_{2}(\lambda;H_f,G_f,y)+\left(y_1+y_2\right) T_{2}(\lambda;H_f,y)}}\right)
		\end{equation*}
	\end{Theorem}
	Theorem \ref{thm2} provides an explicit expression for the asymptotic misclassification rate, influenced by $p/n_1, p/n_2$, and the tuning parameter $\bm{\theta}$. To reduce the bias caused by unequal sample sizes, we present the bias-corrected results in the next subsection.
	\subsection{Bias correction}\label{sub3.2}
	When the sample sizes are different, the estimation bias in the intercept part of SEDA will lead to different misclassification rates. Since $\Phi(\cdot)$ is strictly convex on $(-\infty,0)$, we can reduce the misclassification rate by removing the unnecessary term $\left(y_1-y_2\right)T_1(\lambda;H_f,y)$. To this end, we consider the following classifier,
	\begin{align}
		D(\bm{x})=\I\left\{\left(\bm{x}-\frac{\bar{\bm{x}}_1+\bar{\bm{x}}_2}{2}\right)\trans \left(\bm{S}_n+\lambda \bm{\mathcal{I}}\right)^{-1}\left(\bar{\bm{x}}_1-\bar{\bm{x}}_2\right)+\alpha>0\right\}. 
	\end{align}
	By the Proposition 2 in \cite{mai2012Direct}, when the classification direction is $\left(\bm{S}_n+\lambda \bm{\mathcal{I}}\right)^{-1}\\\left(\bar{\bm{x}}_1-\bar{\bm{x}}_2\right)$, the optimal intercept corresponding to minimum misclassification rate is
	\begin{align*}
		\alpha_0=-\frac{1}{2}\left(\bm{\mu}_1+\bm{\mu}_2\right)\trans\left(\bm{S}_n+\lambda \bm{\mathcal{I}}\right)^{-1}\left(\bar{\bm{x}}_1-\bar{\bm{x}}_2\right),
	\end{align*}
	while for SEDA the intercept is set to be
	\begin{align*}
		\alpha_1=-\frac{1}{2}\left(\bar{\bm{x}}_1+\bar{\bm{x}}_2\right)\trans\left(\bm{S}_n+\lambda \bm{\mathcal{I}}\right)^{-1}\left(\bar{\bm{x}}_1-\bar{\bm{x}}_2\right).
	\end{align*}
	Then, we can calculate the difference between $\alpha_0$ and $\alpha_1$ to adjust the intercept term.
	\begin{align*}
		\alpha:=\alpha_0-\alpha_1=&\frac{1}{2n_{1}} \bm{w}_{1}^{T} \bm{\Sigma}^{\frac{1}{2}}\left(\bm{S}_n+\lambda \bm{\mathcal{I}}\right)^{-1} \bm{\Sigma}^{\frac{1}{2}} \bm{w}_{1}-\frac{1}{2n_{2}} \bm{w}_{2}^{T} \bm{\Sigma}^{\frac{1}{2}}\left(\bm{S}_n+\lambda \bm{\mathcal{I}}\right)^{-1} \bm{\Sigma}^{\frac{1}{2}} \bm{w}_{2}\\&-\frac{1}{2}\left(\frac{1}{\sqrt{n_1}}\bm{\Sigma}^{\frac{1}{2}} \bm{w}_1+\frac{1}{\sqrt{n_2}}\bm{\Sigma}^{\frac{1}{2}} \bm{w}_2\right)\trans\left(\bm{S}_n+\lambda\bm{\mathcal{I}}\right)^{-1}\left(\bm{\mu}_1-\bm{\mu}_2\right),
	\end{align*}
	where $\bm{w}_1,\bm{w}_2\sim N(0,\bm{I}_p)$ are independent with $\bm{S}_n$. Since $\alpha$ depends on the population covariance matrix $\bm{\Sigma}$, which is unknown in practice, we find an asymptotically equivalent 
	\begin{align}\label{alpha}
		\widehat{\alpha}=\left(\frac{p}{2n_1}-\frac{p}{2n_2}\right)\frac{1-\frac{1}{p}\tr\left[\frac{1}{\lambda}\bm{S}_n\bm{\bm{\mathcal{I}}}^{-1}+\bm{I}_p\right]^{-1}}{1-\frac{p}{n}+\frac{1}{n}\tr\left[\frac{1}{\lambda}\bm{S}_n\bm{\bm{\mathcal{I}}}^{-1}+\bm{I}_p\right]^{-1}}. 
	\end{align}
	The derivation of $\widehat{\alpha}$ is deferred to the Appendix. Based on the above, we propose the corrected SEDA classifier
	\begin{align}
		D_{\text{SEDA}}^c(\bm{x})=\I\left\{\left(\bm{x}-\frac{\bar{\bm{x}}_1+\bar{\bm{x}}_2}{2}\right)\trans\left(\bm{S}_n+\lambda \bm{\mathcal{I}}\right)^{-1}\left(\bar{\bm{x}}_1-\bar{\bm{x}}_2\right)+\widehat{\alpha}>0\right\}. 
	\end{align}
	Then the misclassification rate of the corrected SEDA is
	\begin{align*}
		\r^c_{\text{SEDA}}(\bm{\theta})=\frac{1}{2}\sum_{i=1}^2 \Phi\left(\frac{(-1)^i\left[\left(2\bm{\mu}_i-\bar{\bm{x}}_1-\bar{\bm{x}}_2\right)\trans \left(\bm{S}_n+\lambda \bm{\mathcal{I}}\right)^{-1}\left(\bar{\bm{x}}_1-\bar{\bm{x}}_2\right)+2\widehat{\alpha}\right] }{2\sqrt{\left(\bar{\bm{x}}_1-\bar{\bm{x}}_2\right)\trans \left(\bm{S}_n+\lambda \bm{\mathcal{I}}\right)^{-1}\bm{\Sigma} \left(\bm{S}_n+\lambda \bm{\mathcal{I}}\right)^{-1} \left(\bar{\bm{x}}_1-\bar{\bm{x}}_2\right)}}\right).
	\end{align*}
	
	Similar to Theorem \ref{thm2}, we obtain the asymptotic misclassification rate for the corrected SEDA as the following theorem.
	\begin{Corollary}[Asymptotic misclassification rate for corrected SEDA]\label{cor2.4}
		Under the conditions of Theorem \ref{thm2}, for the corrected SEDA,
		almost surely
		\begin{equation}
			\r^c_{\text{SEDA}}(\bm{\theta})\rightarrow\Phi\left(-\frac{U_1(\lambda;H_f,G_f,y)}{2 \sqrt{U_2(\lambda;H_f,G_f,y)+\left(y_1+y_2\right) T_2(\lambda;H_f,y)}}\right)
		\end{equation}	
	\end{Corollary}
	Again, since $\Phi(\cdot)$ is strictly convex on $(-\infty, 0)$, it can be concluded that the asymptotic misclassification rate of bias-corrected SEDA is smaller than that of SEDA. 
	\subsection{Selection of parameters}\label{sub3.3}
	The performance of SEDA depends critically on the choice of $\bm{\theta}$. Although methods like cross-validation are widely used for parameter selection, they can be computationally demanding when both $p$ and $n$ are large. To address this issue, we derive a direct estimator for the optimal parameter.
	
	By Corollary \ref{cor2.4}, the optimal $\bm{\theta}$ with minimum error rate is
	\begin{align*}
		\bm{\theta}_0\in \mathop{\arg\max}\limits_{\bm{\theta}}\frac{U_1^2(\lambda;H_f,G_f,y)}{U_2(\lambda;H_f,G_f,y)+\left(y_1+y_2\right)T_2(\lambda;H_f,y)}.
	\end{align*}
	Although the structure of the non-spiked part of $\Sigma$ is unobservable, direct estimates of the optimal parameters can still be obtained under additional conditions. Specifically, we consider the setup of the simple spiked model i.e., $s_{r_1+1}=s_{r_1+2}=\dots=s_{p-r_2}=\sigma^2$, noting that this condition is necessary only for estimating $\|\bm{\mu}\|$, and that it is relaxed for the other parts to $\langle \bm{\mu},\bm{v}_ {r_1+1}\rangle=\langle \bm{\mu},\bm{v}_{r_1+2}\rangle=\dots=\langle \bm{\mu},\bm{v}_{p-r_2}\rangle$. For simplicity, we treat $\sigma^2,s_j$ and $\chi_j(j)$ as known, since their consistent estimates are already given in \cite{jiang2021Generalized} and \cite{pu2024Asymptotic}. Then, we can obtain the following consistent estimates. The detailed calculation process is moved to the Appendix.
	\begin{align}
		\widehat{T}_2:=&\frac{1-\lambda\widehat{m}}{\left(1-\widehat{y}+\widehat{y}\lambda\widehat{m}\right)^3}-\frac{\lambda\widehat{m}-\lambda^2\widehat{m}'}{\left(1-\widehat{y}+\widehat{y}\lambda\widehat{m}\right)^4}\xrightarrow{a.s.}T_2(\lambda;H_f,y),\label{T2}\\
		\widehat{U}_1:=&\sum_{j\in\mathbb{J}}\beta_j\frac{\widetilde{s}_j}{\widetilde{s}_j\left(1-\widehat{y}+\widehat{y}\lambda\widehat{m}\right)+\lambda}+\left(\gamma-\sum_{j\in\mathbb{J}}\beta_j\right)\frac{1-\lambda\widehat{m}}{1-\widehat{y}+\widehat{y}\lambda\widehat{m}}\notag\\&\xrightarrow{a.s.}U_1(\lambda;H_f,G_f,y) ,\label{U1}\\
		\widehat{U}_2:=&\left(1+\widehat{y}\widehat{m}_1\right)\left\{\sum_{j\in\mathbb{J}}\beta_j\left(\frac{\widetilde{s}_j}{\widetilde{s}_j\left(1-\widehat{y}+\widehat{y}\lambda\widehat{m}\right)+\lambda}\right)^2+\left(\gamma-\sum_{j\in\mathbb{J}}\beta_j\right)\widehat{T}_2\right\}\notag\\&\xrightarrow{a.s.}U_2(\lambda;H_f,G_f,y) ,\label{U2}
	\end{align}
	where 
	\begin{gather*}
		\widehat{m}=\frac{1}{p}\tr\left[\bm{S}_n\bm{\bm{\mathcal{I}}}^{-1}+\lambda \bm{I}_p\right]^{-1},\quad\widehat{m}'=\frac{1}{p}\tr\left[\bm{S}_n\bm{\bm{\mathcal{I}}}^{-1}+\lambda \bm{I}_p\right]^{-2},\\
		\widetilde{s}_j=\left[1+\frac{\ell_j}{1-\ell_j}\chi_j(j)\right]s_j,\quad\widehat{m}_1=\frac{1}{\widehat{y}\left(1-\widehat{y}+\widehat{y}\lambda\widehat{m}\right)}-\frac{\widehat{y}\lambda\left(\widehat{m}-\lambda\widehat{m}'\right)}{\widehat{y}\left(1-\widehat{y}+\widehat{y}\lambda\widehat{m}\right)^2}-\frac{1}{\widehat{y}},\\
		\beta_j=\chi_j(j)\left\langle\bar{\bm{x}}_1-\bar{\bm{x}}_2,\bm{u}_j\right\rangle^2/s_j,\quad\gamma=\sum_{j\in\mathbb{J}}\left(1-s_j/\sigma^2\right)\beta_j+\|\bar{\bm{x}}_1-\bar{\bm{x}}_2\|^2/\sigma^2-\widehat{y}_1-\widehat{y}_2,
	\end{gather*}
	with $\widehat{y}_i=p/n_i,i=1,2$ and $\widehat{y}=p/n$. Then, the estimation of the optimal parameters is given by
	\begin{align}\label{theta}
		\widehat{\bm{\theta}}_0\in \mathop{\arg\max}\limits_{\bm{\theta}}\frac{\widehat{U}_1^2}{\widehat{U}_2+(\widehat{y}_1+\widehat{y}_2)\widehat{T}_2}.
	\end{align}
	
	We derive a theoretical estimate for the optimal parameters in simplified scenarios. However, extending this analysis to general structures presents significant theoretical challenges. Consequently, we focus our theoretical treatment on the basic case and defer the investigation of complex settings to numerical experiments. In Section \ref{sec4}, we evaluate our proposed parameter estimation method against cross-validation approaches.
	
	\section{Simulation}\label{sec4}
	In this section, we conducted several simulations to validate our results and discussed the performance of the SEDA classifier. For comparison, we also included RLDA, SRLDA, and SIDA.
	
	We independently generate the training samples $\bm{x}_{1,1},\bm{x}_{1,2},\dots,\bm{x}_{1,n_1}\sim N(\bm{\mu}_1,\bm{\Sigma})$ and $\bm{x}_{2,1},\\\bm{x}_{2,2},\dots,\bm{x}_{2,n_2}\sim N(\bm{\mu}_2,\bm{\Sigma})$. The elements of $\bm{\mu}_1$ are independent and identically distributed from $N(0, 1)$ and $\bm{\mu}_2 = \bm{0}_p$. The covariance matrices are generated as follows:
	\begin{enumerate}
		\item[] Case 1:\quad$\bm{\Sigma}=\text{diag}(0.01,0.05,1,\dots,1,10)$;
		\item[] Case 2:\quad$\bm{\Sigma}=\text{diag}(0.01,0.05,1,\dots,1,5,\dots,5,20)$;
		\item[] Case 3:\quad$\bm{\Sigma}=(\bm{\Sigma}_{ij})_{p\times p}$, $\bm{\Sigma}_{ij}=\mathbb{I}(i=j)-1/p\cdot\mathbb{I}(i\neq j)$.
	\end{enumerate}
	As a benchmark, adjust $\bm{\mu}_1=(\mu_{1,1},\mu_{1,2},\dots,\mu_{1,p})$ such that the true Bayes error rate reaches 10$\%$. 
	
	Case 1 is a simple case of homoscedasticity and independence, which we use as a benchmark. We will illustrate the limitations of SIDA and SRLDA with Case 2 and Case 3, where Case 2 does not satisfy homoscedasticity and Case 3 is a case of strong correlation.
	\subsection{Theoretical and empirical error rate of RLDA}
	We will examine the consistency between the theoretical error rates and empirical error rates of RLDA and SEDA. For convenience, we use the settings of Case 1 and set $\lambda=0.1,\ell_1=\ell_2=0.5 $ and $\ell_p=-1$. The data dimension $p$ ranges from 20 to 200 and the ratio $p/n$ is fixed as 0.5, 1, 1.5. To maintain structural consistency, fix $\mu_{1,1}=\mu_{1,2}=\mu_{1,p}=0.1$. Figure \ref{fig1} shows the box plot of the error rates of two classifiers based on 1000 repeated experiments. The vertical axis represents the percentage of empirical classification error rate, and the horizontal axis represents the dimension $p$. From Figure \ref{fig1}, we observe that the empirical error rates converge to the theoretical results, which is consistent with our conclusions.
	\begin{figure}
		\centering
		\begin{minipage}{0.30\textwidth}
			\includegraphics[width=\textwidth]{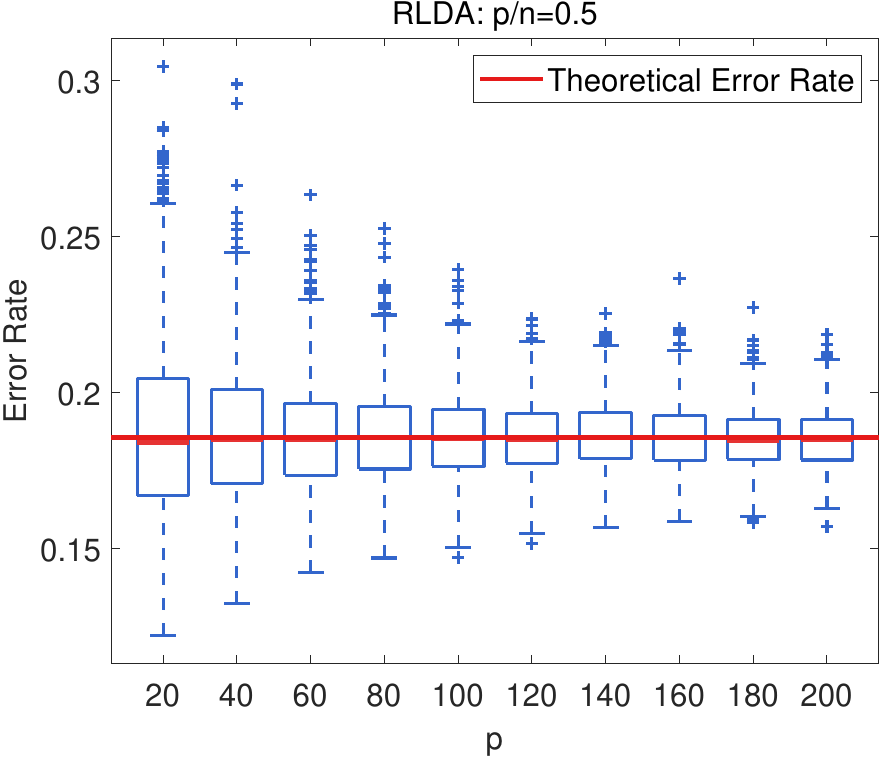}
		\end{minipage}\quad
		\begin{minipage}{0.30\textwidth}
			\includegraphics[width=\textwidth]{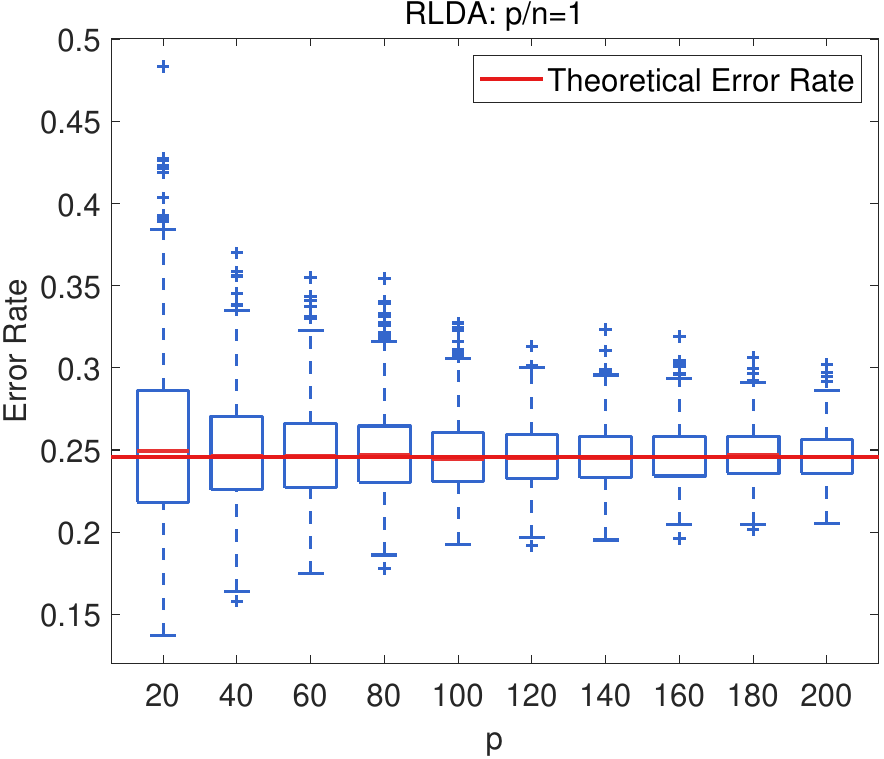}
		\end{minipage}\quad
		\begin{minipage}{0.30\textwidth}
			\includegraphics[width=\textwidth]{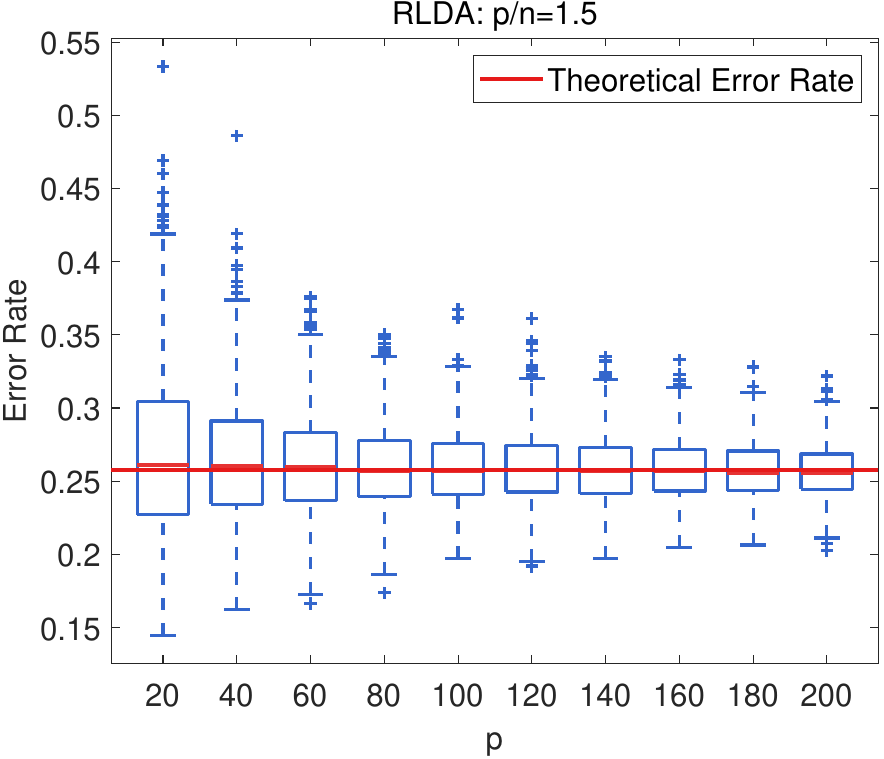}
		\end{minipage}
		\begin{minipage}{0.30\textwidth}
			\includegraphics[width=\textwidth]{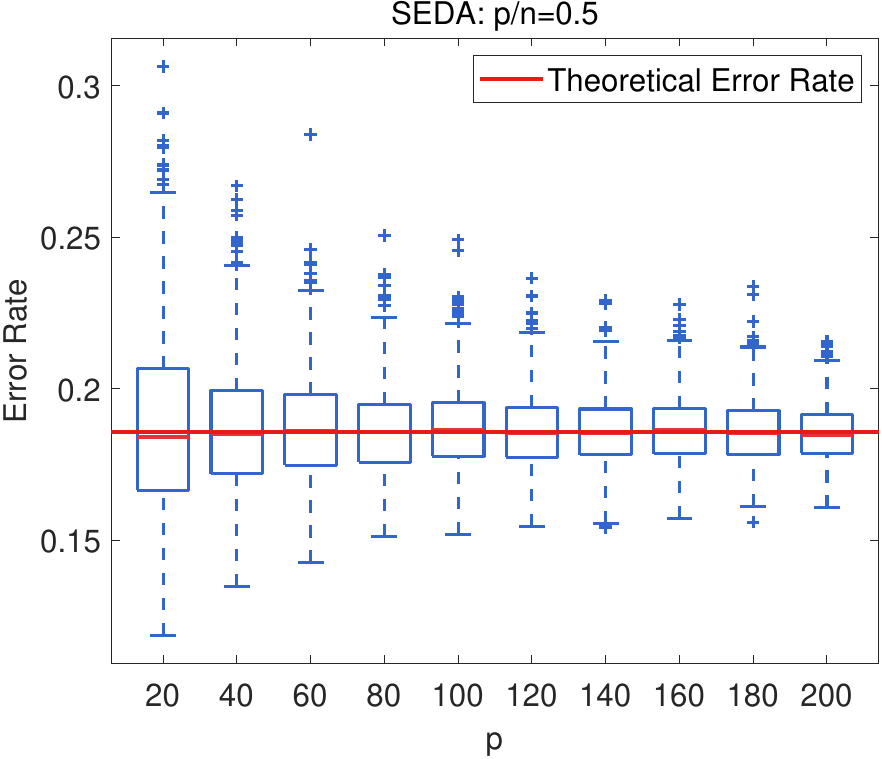}
		\end{minipage}\quad
		\begin{minipage}{0.30\textwidth}
			\includegraphics[width=\textwidth]{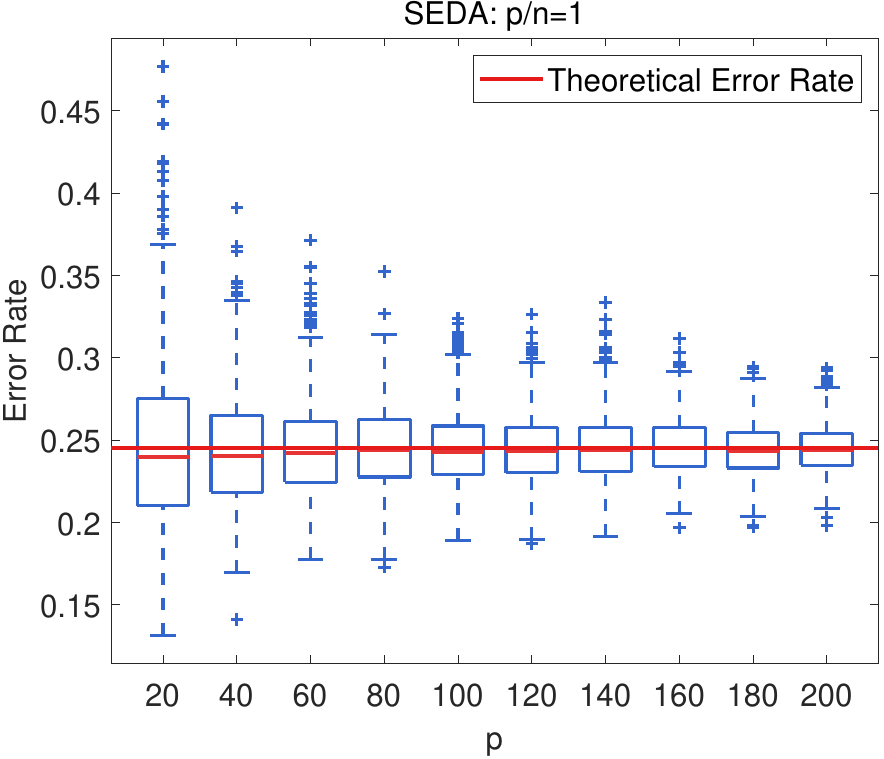}
		\end{minipage}\quad
		\begin{minipage}{0.30\textwidth}
			\includegraphics[width=\textwidth]{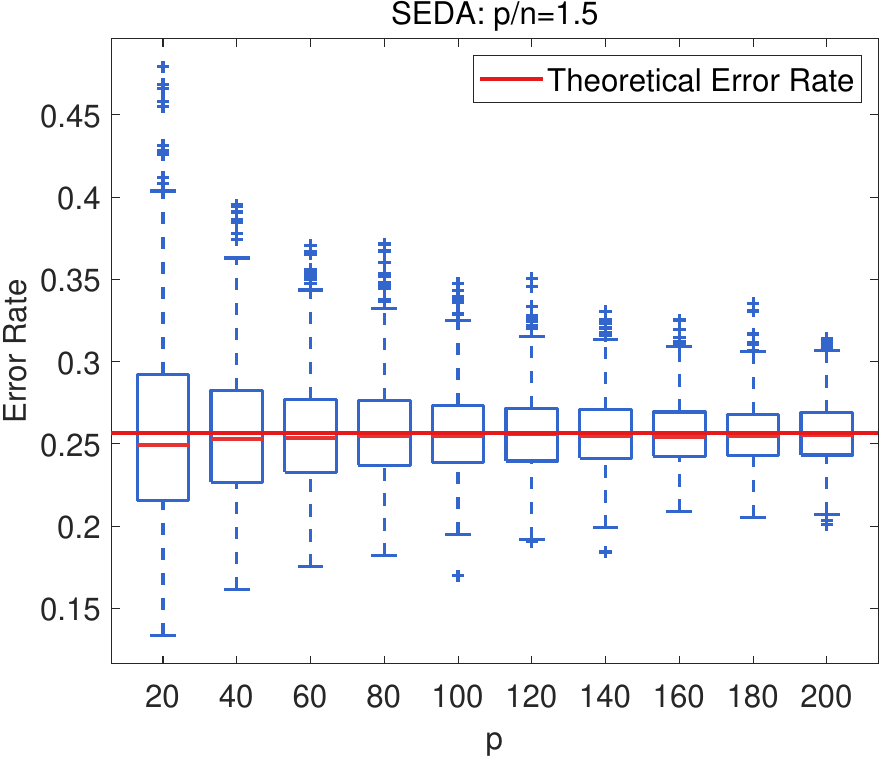}
		\end{minipage}
		\caption{Consistency of theoretical and empirical error rate.}
		\label{fig1}
	\end{figure}
	
	\subsection{Performance of Classifiers}
	In this subsection, we compare the misclassification rates of RLDA, SEDA, SIDA, SRLDA, and SEDA with optimal parameters (opSEDA) under different cases. Since \cite{li2025Highdimensional} did not provide a method for parameter selection for SIDA, we used 5-fold cross-validation for parameter tuning. For SEDA, we simultaneously compared cross-validation with our optimal parameter selection method. For each case, we fixed $p = 100$. Figure \ref{fig2} shows the empirical misclassification rate based on 1000 repetitions, where the testing sample size is set at 1000. The vertical axis represents the percentage of empirical classification error rate, and the horizontal axis represents the training sample size $n$.
	\begin{figure}
		\centering
		\begin{minipage}{0.96\textwidth}
			\includegraphics[width=\textwidth]{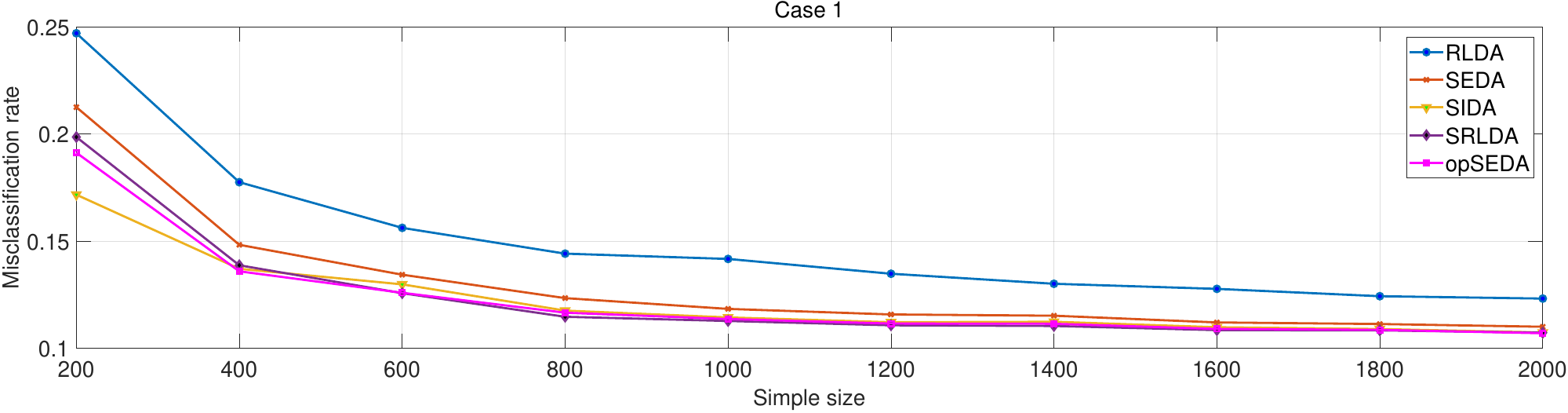}
		\end{minipage}
		\begin{minipage}{0.96\textwidth}
			\includegraphics[width=\textwidth]{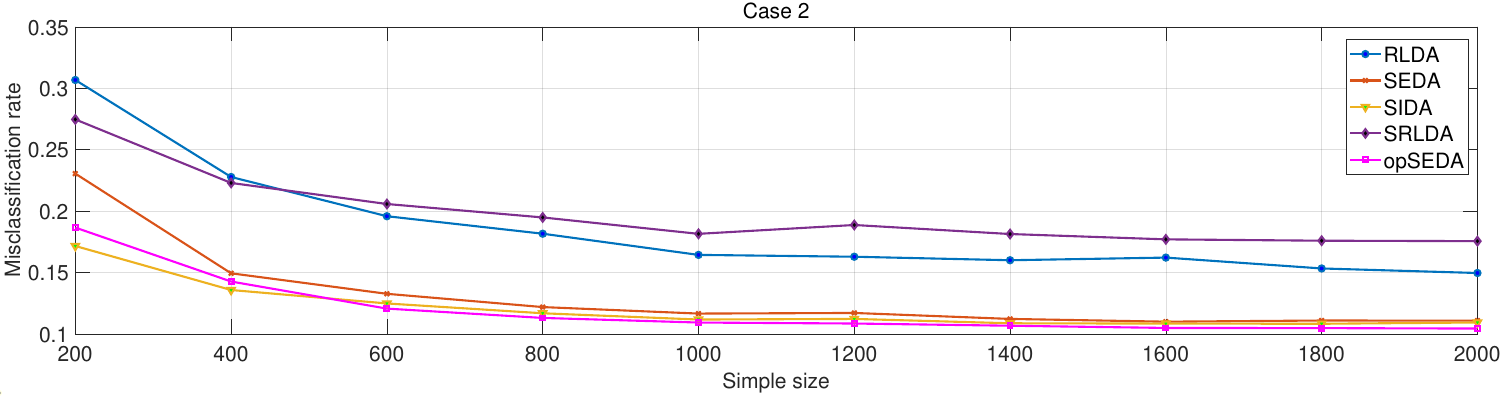}
		\end{minipage}
		\begin{minipage}{0.96\textwidth}
			\includegraphics[width=\textwidth]{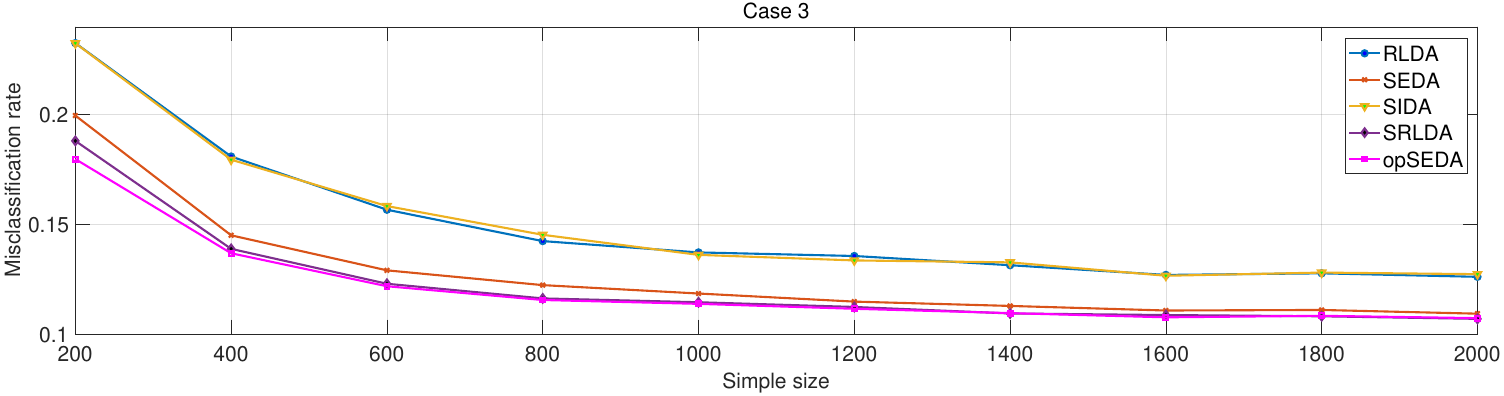}
		\end{minipage}
		\caption{Comparison of misclassification rates of RLDA, SEDA, SIDA, SRLDA, and opSEDA under different cases.}
		\label{fig2}
	\end{figure}
	
	From these simulation results, under the simple setup of Case 1, SEDA, SIDA, and SRLDA are all significantly better than the traditional RLDA since RLDA does not utilize structural information. For SRLDA, there is a noticeable decrease in accuracy when the assumptions of its model are not met in Case 2. For SIDA, its classification accuracy reaches its optimum when the dimensions of the samples are mutually independent. This is because when dealing with independent data, SIDA can directly normalize all eigenvalues of the population covariance matrix. However, under the strong correlation setting in Case 3, the performance of SIDA is as poor as that of RLDA. This aligns with our expectations; both SRLAD and SIDA have limitations in their application scenarios. In contrast, SEDA has demonstrated excellent performance under various conditions. And in all cases, our optimal parameter selection method outperforms cross-validation.
	\subsection{Bias correction}
	In this subsection, we compare the performance of SEDA and corrected SEDA under the setting of Case 1, using the classifier with the optimal intercept $\alpha_0$ as the benchmark. Set $\lambda=0.1,\ell_1=\ell_2=0.5$ and $\ell_p=-1$. Data dimension $p\in\{100,200,400\}$, sample size $n=n_1+n_2=200$, and $n_1$ ranges from 30 to 170. The vertical axis represents the percentage of empirical classification error rate, and the horizontal axis represents the training sample size $n_1$. The testing sample size is set at 100, and the simulation times are 1000. Figure \ref{fig3} shows that when the sample sizes are unequal, the corrected SEDA has a lower misclassification rate than SEDA and is close to the classifier with optimal intercept, indicating that our bias correction is effective and close to optimal. 
	\begin{figure}
		\centering
		\begin{minipage}{0.30\textwidth}
			\includegraphics[width=\textwidth]{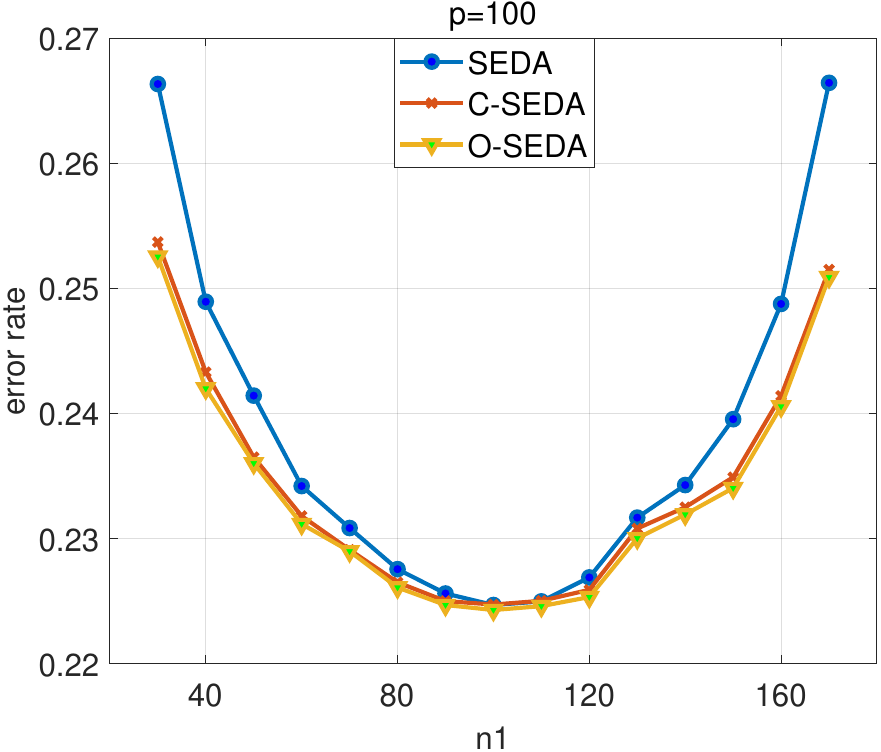}
		\end{minipage}\quad
		\begin{minipage}{0.30\textwidth}
			\includegraphics[width=\textwidth]{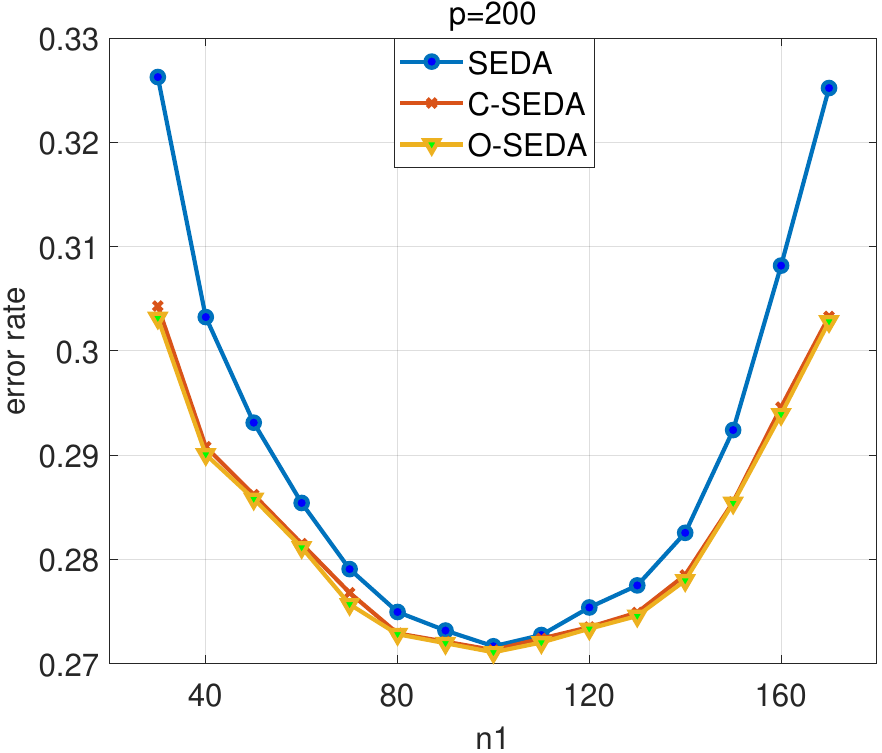}
		\end{minipage}\quad\begin{minipage}{0.30\textwidth}
			\includegraphics[width=\textwidth]{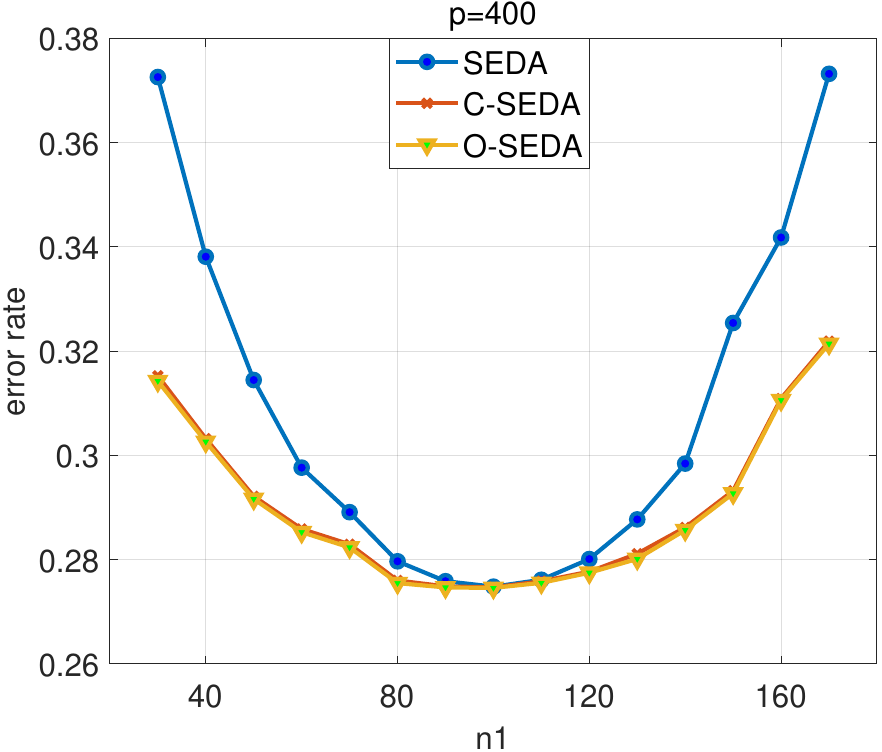}
		\end{minipage}
		\caption{Simulations for SEDA, bias corrected SEDA (C-SEDA) and SEDA with optimal intercept (O-SEDA).}
		\label{fig3}
	\end{figure}
	\section{Real data analysis}\label{sec5}
	In this section, we evaluate the performance of our proposed SEDA classifier using two benchmark datasets. The first dataset is the MNIST Handwritten Digits Database obtained from the UCI Machine Learning Repository, which comprises 70,000 grayscale images of handwritten digits (0-9) with a resolution of 28 × 28 pixels. The second dataset is the CIFAR-10 dataset, which contains 60,000 color images across 10 classes, including airplanes, cars, birds, cats, deer, dogs, frogs, horses, ships, and trucks, with a resolution of 32 × 32 pixels.
	
	In practical applications, more attention has been paid to the dimensionality reduction performance of LDA algorithms in multiple-class problems. Therefore, we first give the extension of the SEDA algorithm to multiple-class problems in the first subsection and examine its effect on real data in Subsection \ref{sub5.3}.
	\subsection{Extension to multiple-class SEDA}\label{sub5.1}
	In this subsection, we discussed the extension of SEDA to the $K$-class LDA. More specifically, we consider the following data setting. Suppose we have $K$ different classes, each with samples drawn from a $p$-dimensional multivariate normal distribution with mean vector $\bm{\mu}_k$ and covariance matrix $\bm{\Sigma}$, where $k = 1, 2, \dots, K$. We randomly select $n_k$ samples from the k-th class, that is $C_k:  \bm{x}_{k,1},\bm{x}_{k,2},\dots,\bm{x}_{k,n_k}\sim N(\bm{\mu}_k,\bm{\Sigma})$. The total sample size is $n=\sum_{k=1}^{K}n_k$.
	
	The goal of $K$-class LDA is to find a subspace with a maximum dimension of $(K-1)$ that maximizes the inter-class distance and minimizes the intra-class distance. In other words, the optimal projection matrix $\bm{W}^*$ is
	\begin{align}\label{W}
		\bm{W}^*=\mathop{\arg\max}\limits_{\bm{W}}\frac{\tr \left(\bm{W}\trans\bm{\Sigma}_b \bm{W}\right)}{\tr \left(\bm{W}\trans\bm{\Sigma} \bm{W}\right)},
	\end{align}
	in which
	\begin{align*}
		\bm{\Sigma}_b=\sum_{k=1}^{K}\left(\bm{\mu}_k-\bar{\bm{\mu}}\right)\left(\bm{\mu}_k-\bar{\bm{\mu}}\right)\trans,
	\end{align*}
	with $\bar{\bm{\mu}}=\frac{1}{K}\sum_{k=1}^{K}\bm{\mu}_k$. The optimal solution $\bm{W}^*$ of \eqref{W} consists of the eigenvectors corresponding to the $(K-1)$ largest eigenvalues of $\bm{\Sigma}^{-1}\bm{\Sigma}_b$.
	
	For SEDA, we use $\left(\bm{S}_w+\lambda \bm{\mathcal{I}}\right)$ in \eqref{L} as an estimate of $\bm{\Sigma}$, where $\bm{u}_j$ is the eigenvector corresponding to the $j$-th largest eigenvalue of the within-class scatter matrix $\bm{S}_w$,
	\begin{align*}
		\bm{S}_w=\frac{1}{n}\sum_{k=1}^{K}\sum_{j=1}^{n_k}\left(\bm{x}_{k,j}-\bar{\bm{x}}_k\right)\left(\bm{x}_{k,j}-\bar{\bm{x}}_k\right)\trans,
	\end{align*}
	where $\bar{\bm{x}}_k=\frac{1}{n_k}\sum_{j=1}^{n_k}\bm{x}_{k,j}$. And $\bm{\Sigma}_b$ is estimated by the between-class scatter matrix $\bm{S}_b$,
	\begin{align*}
		\bm{S}_b=\sum_{k=1}^{K}\frac{n_k}{n}\left(\bm{x}_k-\bar{\bm{x}}\right)\left(\bm{x}_k-\bar{\bm{x}}\right)\trans,
	\end{align*}
	where $\bar{\bm{x}}=\frac{1}{n}\sum_{k=1}^{K}\sum_{j=1}^{n_k}\bm{x}_{k,j}$. Then, the estimation of the optimal parameters is given by the joint error rate function
	\begin{align}
		\widehat{\bm{\theta}}=\mathop{\arg\max}\limits_{\bm{\theta}}\sum_{i\neq j}R_{ij}^{\text{SEDA}},
	\end{align}
	where $R_{ij}^{\text{SEDA}}$ is the estimated value of the asymptotic misclassification between $C_i$ and $C_j$. And we can use the expressions given in Subsection \ref{sub3.3} to obtain it.
	\subsection{The case of two classes}
	In this subsection, we evaluate the binary classification performance of SEDA using the MNIST dataset. We select handwritten digits 3 and 8 as the target classes and conduct experiments with different sample sizes $n \in \{300, 600, 900\}$ with the ratio $n_1/n_2=0.5$. Additionally, we evaluate the combined performance of SEDA after kernel transformation and PCA dimensionality reduction using the same experimental setup. Specifically, we set the PCA dimensionality reduction rate to 0.5; we chose a polynomial kernel function and set the degree to 2. Since SIDA can be considered a standardization method, we use the standardized data in classifiers other than SIDA. Table \ref{table1} shows the accuracy of several classifiers under different scenarios. It can be seen that SEDA is the best in terms of both direct and combined performance. In fact, \cite{li2025Spectrallycorrected}'s experiments showed that the MNIST dataset contains a large number of spiked eigenvalues, and it is difficult to effectively adjust them with just standardization. Especially under kernel transformation, the advantages of SRLDA and SEDA are more pronounced. Secondly, SRLDA, due to its overly simplistic model assumptions, has lost a significant amount of sample information, resulting in overall performance that is lower than that of SEDA, especially when $p>n$.
	\begin{table}[h]
		\makeatletter
		\renewcommand{\@makecaption}[2]{
			\vskip\abovecaptionskip
			#1: #2\par
			\vskip\belowcaptionskip
		}
		\makeatother
		\caption{Comparison of the performance of RLDA, SIDA, SRLDA, and SEDA using the MNIST dataset of handwritten digits 3 and 8 under different sample sizes and data processing methods.}
		\label{table1}
		\setlength{\extrarowheight}{3pt}
		\normalsize
		\begin{tabularx}{\textwidth}{
				>{\raggedleft\arraybackslash\hsize=0.3\hsize}X 
				>{\centering\arraybackslash\hsize=0.5\hsize}X 
				>{\centering\arraybackslash\hsize=0.5\hsize}X 
				>{\centering\arraybackslash\hsize=0.5\hsize}X 
				>{\centering\arraybackslash\hsize=0.5\hsize}X
			}
			\toprule
			& \textbf{RLDA} & \textbf{SIDA} & \textbf{SRLDA} & \textbf{SEDA} \\
			\midrule
			\addlinespace[0.1cm]
			&\multicolumn{4}{c}{ \textbf{Unprocessed} }\\
			\addlinespace[0.1cm]
			n=300 & 0.622 & 0.765 & 0.645 & 0.779 \\
			600 & 0.726 & 0.848 & 0.690 & 0.856 \\
			900 & 0.757 & 0.884 & 0.794 & 0.899 \\
			\addlinespace[0.1cm]
			&\multicolumn{4}{c}{ \textbf{PCA dimensionality reduction} }\\
			\addlinespace[0.1cm]
			300 & 0.645 & 0.767 & 0.727 & 0.780 \\
			600 & 0.738 & 0.850 & 0.822 & 0.857 \\
			900 & 0.759 & 0.885 & 0.875 & 0.899 \\
			\addlinespace[0.1cm]
			&\multicolumn{4}{c}{ \textbf{Kernel transformation} }\\
			\addlinespace[0.1cm]
			300 & 0.776 & 0.879 & 0.901 & 0.912 \\
			600 & 0.828 & 0.908 & 0.922 & 0.936 \\
			900 & 0.852 & 0.919 & 0.925 & 0.937 \\
			\bottomrule
		\end{tabularx}
	\end{table}
	\subsection{The case of multiple classes}\label{sub5.3}
	This subsection applies the SEDA method to feature selection and extraction in multi-class classification problems to test its dimensionality reduction effect. Specifically, we choose the CIFAR-10 dataset as the test dataset and select the HOG feature extraction method, with an extraction dimension of 324. The dataset consisting of 60,000 images is partitioned into 10 subsets, each containing 5,000 training images and 1,000 test images. Under the same conditions, RLDA, SIDA, SRLDA, and SEDA each reduced the data dimensions to 9. Using the data before dimensionality reduction as a benchmark, we compared the performance of the four dimensionality reduction methods under the kernel SVM classifier. Table \ref{table2} shows the dimensionality reduction effects of several algorithms across ten subsets. It can be seen that SEDA achieves significantly higher accuracy across different classes compared to other methods, while the dimensionality reduction loss is consistently controlled within 0.02.
	\begin{table}[h]
		\makeatletter
		\renewcommand{\@makecaption}[2]{
			\vskip\abovecaptionskip
			#1: #2\par
			\vskip\belowcaptionskip
		}
		\makeatother
		\caption{Comparison of the dimensionality reduction effects of RLDA, SRLDA, and SEDA using the CIFAR-10 dataset with kernel transformation.}
		\label{table2}
		\setlength{\extrarowheight}{3pt}
		\normalsize
		\begin{tabularx}{\textwidth}{
				>{\raggedleft\arraybackslash\hsize=0.3\hsize}X 
				>{\centering\arraybackslash\hsize=0.5\hsize}X 
				>{\centering\arraybackslash\hsize=0.5\hsize}X 
				>{\centering\arraybackslash\hsize=0.5\hsize}X 
				>{\centering\arraybackslash\hsize=0.5\hsize}X
			}  
			\toprule
			& \textbf{Naive} & \textbf{RLDA} & \textbf{SRLDA} & \textbf{SEDA} \\
			\midrule
			Subset 1 & 0.419 & 0.308 & 0.387 & 0.401 \\
			Subset 2 & 0.455 & 0.353 & 0.433 & 0.438 \\
			Subset 3 & 0.434 & 0.322 & 0.404 & 0.416 \\
			Subset 4 & 0.443 & 0.341 & 0.408 & 0.427 \\
			Subset 5 & 0.389 & 0.275 & 0.355 & 0.379 \\
			Subset 6 & 0.376 & 0.269 & 0.351 & 0.362 \\
			Subset 7 & 0.407 & 0.295 & 0.375 & 0.389 \\
			Subset 8 & 0.385 & 0.268 & 0.350 & 0.376 \\
			Subset 9 & 0.442 & 0.339 & 0.409 & 0.422 \\
			Subset 10 & 0.382 & 0.281 & 0.345 & 0.367 \\
			\bottomrule
		\end{tabularx}
	\end{table}
	\section{Conclusion}\label{sec6}
	This work provides a comprehensive theoretical analysis of regularized linear discriminant analysis (RLDA) and proposes an enhanced classification method based on spectral modification. A precise non-asymptotic approximation of the RLDA misclassification rate is derived, offering new insights into how the underlying data structure, particularly the eigenvectors of the population covariance matrix, affects classification performance. The analysis reveals that overemphasis on eigenvectors associated with small eigenvalues can significantly degrade accuracy, and a practical remedy is to amplify those eigenvalues.
	
	Motivated by these findings, we introduce the Spectral Enhanced Discriminant Analysis (SEDA) classifier, which improves classification by adjusting the spiked eigenvalues of the population covariance matrix. A new theoretical result concerning eigenvectors in random matrix theory is developed, leading to an asymptotic approximation of the SEDA misclassification rate. This theoretical foundation also enables the design of a bias correction scheme and a principled parameter selection strategy, making the classifier more robust and broadly applicable in high-dimensional settings.
	
	Future work will explore several promising directions, including extending the SEDA framework to nonlinear classification settings, developing distributed implementations for large-scale data environments, and applying the method to multi-class and imbalanced scenarios. These directions will further enhance the practical value and scalability of the proposed approach.
	
	\newpage
	
	\section*{Appendix}
	In the subsequent proofs, the letters $c,C> 0$ will be used interchangeably as constants independent of the key equation parameters and may be reused. Furthermore, the variable $\varepsilon>0$ will represent any small positive number, and the variable $D > 0$ will represent any large positive number. The variables $c, C$ may depend on $\varepsilon$ and $D$.
	\subsection*{Proof of Theorem \ref{thm1}}
	Write
	\begin{align*}
		& A_{1 n}=\left(2 \bm{\mu}_1-\bar{\bm{x}}_1-\bar{\bm{x}}_2\right)\trans \left(\bm{S}_n+\lambda \bm{I}_p\right)^{-1}\left(\bar{\bm{x}}_1-\bar{\bm{x}}_2\right), \\
		& A_{2 n}=-\left(2 \bm{\mu}_2-\bar{\bm{x}}_1+\bar{\bm{x}}_2\right)\trans \left(\bm{S}_n+\lambda \bm{I}_p\right)^{-1}\left(\bar{\bm{x}}_1-\bar{\bm{x}}_2\right), \\ 
		& A_{3 n}=\left(\bar{\bm{x}}_1-\bar{\bm{x}}_2\right)\trans \left(\bm{S}_n+\lambda \bm{I}_p\right)^{-1} \bm{\Sigma}\left(\bm{S}_n+\lambda \bm{I}_p\right)^{-1}\left(\bar{\bm{x}}_1-\bar{\bm{x}}_2\right) .
	\end{align*}
	Since $\bar{\bm{x}}_1 \stackrel{d}{=}\frac{1}{\sqrt{n_1}} \bm{\Sigma}^{\frac{1}{2}} \bm{w}_1+\bm{\mu}_1, \bar{\bm{x}}_2 \stackrel{d}{=} \frac{1}{\sqrt{n_2}} \bm{\Sigma}^{\frac{1}{2}} \bm{w}_2+\bm{\mu}_2$,
	where $\bm{w}_1, \bm{w}_2 \sim N\left(0, \bm{I}_p\right)$ and $\bm{w}_1, \bm{w}_2, \bm{S}_n$ are independent, we have
	\begin{align*}
		A_{1 n}& \stackrel{d}{=} \bm{\mu}\trans \bm{B}_n(\lambda) \bm{\mu}-\frac{2}{\sqrt{n_2}} \bm{\mu}\trans \bm{B}_n(\lambda) \bm{w}_2+\frac{1}{n_2} \bm{w}_2\trans \bm{B}_n(\lambda) \bm{w}_2-\frac{1}{n_1} \bm{w}_1\trans \bm{B}_n(\lambda) \bm{w}_1, \\
		A_{2 n}& \stackrel{d}{=} \bm{\mu}\trans \bm{B}_n(\lambda) \bm{\mu}+\frac{2}{\sqrt{n_1}} \bm{\mu}\trans \bm{B}_n(\lambda) \bm{w}_1+\frac{1}{n_1} \bm{w}_1\trans \bm{B}_n(x) \bm{w}_1-\frac{1}{n_2} \bm{w}_2\trans \bm{B}_n(\lambda) \bm{w}_2, \\
		A_{3 n} &\stackrel{d}{=}\left(\bm{\mu}+\sqrt{\frac{1}{n_1}} \bm{w}_1-\sqrt{\frac{1}{n_2}} \bm{w}_2\right)\trans \bm{B}_n^2(\lambda)\left(\bm{\mu}+\sqrt{\frac{1}{n_1}} \bm{w}_1-\sqrt{\frac{1}{n_2}} \bm{w}_2\right) \\
		&\stackrel{d}{=} \bm{\mu}\trans \bm{B}_n^2(\lambda) \bm{\mu}+2\sqrt{\frac{1}{n_1}+\frac{1}{n_2}} \bm{\mu}\trans \bm{B}_n^2(\lambda) \bm{w}_1+\left(\frac{1}{n_1}+\frac{1}{n_2}\right) \bm{w}_1\trans \bm{B}_n^2(\lambda) \bm{w}_1 .
	\end{align*}
	where $\bm{B}_n(\lambda)=\bm{\Sigma}^{\frac{1}{2}}\left(\bm{S}_n+\lambda \bm{I}_p\right)^{-1}\bm{\Sigma}^{\frac{1}{2}}$.
	
	Next, we use the following lemmas to construct the desired concentration inequality.
	\begin{lemma}\label{lem2}
		Under the conditions of Theorem \ref{thm1}, assuming $\bm{w}\sim N\left(0, \bm{I}_p\right)$ and $\bm{w}$ is independent with $\bm{S}_n$, we have
		\begin{align*}
			P\left( \left| \frac{1}{\sqrt{n_j}} \bm{\mu}\trans \bm{B}_n^k(\lambda) \bm{w} \right|  \geq  n^{-\frac{1-\varepsilon}{2}}\right) \leq C e^{-cn^\varepsilon},\quad j,k=1,2.
		\end{align*}
	\end{lemma}
	\begin{proof}
		For any $\bm{A} \in \mathbb{R}^{p \times p}$ such that $\|\bm{A}\| \leq C$, $f: \mathbb{R}^p \rightarrow \mathbb{R}, \bm{w} \mapsto \bm{\mu}\trans \bm{A} \bm{w}$ is $C $-Lipschitz, because for all $\bm{x},\bm{y} \in \mathbb{R}^p$,
		$$
		\|f(\bm{x})-f(\bm{y})\|=\|\bm{\mu}\trans \bm{A} \bm{x}-\bm{\mu}\trans \bm{A} \bm{y}\| \leq\|\bm{\mu}\| \|\bm{A}\| \|\bm{x}-\bm{y}\| \leq C \|\bm{x}-\bm{y}\| .
		$$
		Then, the maps $\bm{w} \mapsto \bm{\mu}\trans \bm{B}_n^k(\lambda) \bm{w}$ are Lipshitz with parameter $C$. By the \emph{Gaussian concentration inequality} for Lipschitz functions, the lemma is proven.
	\end{proof}
	\begin{lemma}\label{lem3}
		Under the conditions of Theorem \ref{thm1}, we have
		\begin{align*}
			P\left( \left| \bm{\mu}\trans \bm{B}_n(\lambda) \bm{\mu}-\bm{\mu}\trans \bm{\mathcal{B}} \bm{\mu} \right|  \geq  n^{-\frac{1-\varepsilon}{2}}\right) \leq C n^{-D},
		\end{align*}
		and
		\begin{align}
			P\left( \left| \bm{\mu}\trans \bm{B}_n^2(\lambda) \bm{\mu}-\left(1-y_n m_{n, 1}(-\lambda)\right) \bm{\mu}\trans \bm{\mathcal{B}}^{2} \bm{\mu} \right|  \geq  n^{-\frac{1-\varepsilon}{2}}\right) \leq C n^{-D},
		\end{align}
		where $\bm{\mathcal{B}}=\bm{\Sigma}\left(\lambda \bm{I}_p+\left(1-y_n+y_n\lambda m(-\lambda;H_n,y_n)\right) \bm{\Sigma}\right)^{-1}$.
	\end{lemma}
	\begin{proof}
		The second inequality was proved in Theorem 5 of \cite{hastie2022Surprises}. We consider the first inequality. Since $\bm{S}_n$ has the same distribution as that of $\bm{S}=\frac{1}{n}\bm{\Sigma}^{\frac{1}{2}}\bm{X}\bm{X}\trans\bm{\Sigma}^{\frac{1}{2}}$, where $\bm{X}=[\bm{X}_1,\dots,\bm{X}_n]$ denotes a $p\times n$ random matrix such that the random vectors $\bm{X}_i$ have the standard multivariate Gaussian distribution. It is convenient to rewrite $\bm{S}_n$ as $\bm{S}$, and introduce the notation $\bar{\bm{S}}=\frac{1}{n}\bm{X}\bm{X}\trans$. For $\Re(\eta)>-1/M$ define
		\begin{align*}
			\mathcal{D}(\eta,\lambda)=&\lambda\bm{\mu}\trans\bm{\Sigma}^{\frac{1}{2}}\left(\bm{S}_n+\lambda \bm{I}_p+\lambda\eta\bm{\Sigma}\right)^{-1}\bm{\Sigma}^{\frac{1}{2}}\bm{\mu}
			=\lambda\bm{\mu}_\eta\trans\left(\bm{\Sigma}_\eta^{\frac{1}{2}}\bar{\bm{S}}\bm{\Sigma}_\eta^{\frac{1}{2}}+\lambda \bm{I}_p\right)^{-1}\bm{\mu}_\eta,
		\end{align*}
		where
		\begin{align*}
			\bm{\Sigma}_\eta=\bm{\Sigma}\left(\bm{I}_p+\eta\bm{\Sigma}\right)^{-1},\quad\bm{\mu}_\eta=\left(\bm{I}_p+\eta\bm{\Sigma}\right)^{-\frac{1}{2}}\bm{\Sigma}^{\frac{1}{2}}\bm{\mu}.
		\end{align*}
		By Eq. (A.21) in \cite{hastie2022Surprises}, we obtain, with
		probability at least $1-Cn^{-D}$
		\begin{align*}
			\left|\mathcal{D}(\lambda,\eta)-\bm{\mu}_\eta\trans\left(\bm{I}_p+r_n\left(-\lambda,\eta\right)\bm{\Sigma}_\eta\right)^{-1}\bm{\mu}_\eta\right|\leq\frac{1}{n^{(1-\epsilon)/2}}.
		\end{align*}
		Here, $r_n=r_n(-\lambda,\eta)$ is defined as the unique solution of
		\begin{align*}
			\frac{1}{r_n}=\lambda+\frac{y_n}{p}\sum_{i=1}^{p}\frac{s_i(\eta)}{1+s_i(\eta)r_n},
		\end{align*}
		where $s_1(\eta)\geq s_2(\eta)\geq\dots\geq s_p(\eta)$ are the eigenvalues of $\bm{\Sigma}_\eta$. By taking $\eta=0$, the proof is completed.
	\end{proof}
	\begin{lemma}\label{lem4}
		Under the conditions of Theorem \ref{thm1}, assuming $\bm{w}\sim N\left(0, \bm{I}_p\right)$ and $\bm{w}$ is independent with $\bm{S}_n$, we have
		\begin{equation}
			\begin{aligned}
				P\left( \left| \frac{1}{n_j} \bm{w}\trans \bm{B}_n^k(\lambda) \bm{w}-y_{jn}T_k(\lambda;H_n,y_n) \right|  \geq  n^{-\frac{1-\varepsilon}{2}}\right) \leq C n^{-D},\quad j,k=1,2.
			\end{aligned}
		\end{equation}
	\end{lemma}
	\begin{proof}
		Since $\|\bm{B}_n^k(\lambda)\|\leq C$, by the \emph{Hanson-Wright inequality}, we have
		\begin{align}
			P\left(\left| \frac{1}{n_j} \bm{w}\trans \bm{B}_n^k(\lambda) \bm{w}-\frac{1}{n_j} \operatorname{tr} \bm{B}_n^k(\lambda)\right|  \geq \frac{1}{2} n^{-\frac{1-\varepsilon}{2}}\right) \leq C e^{-c n^\varepsilon} . \label{11}
		\end{align}	
		Due to the arbitrariness of $\bm{\mu}$, the following inequality can be directly obtained from Lemma \ref{lem3}.
		\begin{align*}
			P\left(\left|\frac{1}{p} \operatorname{tr} \bm{B}_n^k(\lambda)-T_k(\lambda;H_n,y_n)\right| \geq n^{-\frac{1-\varepsilon}{2}}\right)\leq Cn^{-D}.
		\end{align*}
		Combining with \eqref{11}, we have
		\begin{equation}\label{13}
			\begin{aligned}
				& P\left(\left|\frac{1}{n_j} \bm{w}\trans \bm{B}_n^k(\lambda) \bm{w}-y_{jn}T_k(\lambda;H_n,y_n)\right| \geq n^{-\frac{1-\varepsilon}{2}}\right) \\
				\leq &P\left(\left|\frac{1}{n_j} \bm{w}\trans \bm{B}_n^k(\lambda) \bm{w}-\frac{1}{n_j} \operatorname{tr} \bm{B}_n^k(\lambda)\right|+\left|\frac{1}{n_j} \operatorname{tr} \bm{B}_n^k(\lambda)-y_{jn}T_k(\lambda;H_n,y_n)\right| \geq n^{-\frac{1-\varepsilon}{2}}\right) \\
				\leq &P\left(\left|\frac{1}{n_j} \bm{w}\trans \bm{B}_n^k(\lambda) \bm{w}-\frac{1}{n_j}\operatorname{tr}\bm{B}_n^k(\lambda)\right| \geq \frac{1}{2} n^{-\frac{1-\varepsilon}{2}}\right)\\
				+&P\left(\left| \frac{1}{n_j}\operatorname{tr}\bm{B}_n^k(\lambda)-y_{jn}T_k(\lambda;H_n,y_n) \right|\geq \frac{1}{2} n^{-\frac{1-\varepsilon}{2}}\right) \\
				\leq &C n^{-D}  .
			\end{aligned}
		\end{equation}
		The lemma is proven.
	\end{proof}
	Combining the above three lemmas, we conclude that, with probability at least $1-Cn^{-D}$ the following holds:
	\begin{equation}
		\begin{aligned}
			& \left|A_{1 n}-U_{1}(\lambda;H_n,G_n,y_n)-\left(y_{2n}-y_{1n}\right)T_{1}(\lambda;H_n,y_n)\right| \\\leq&\left|\bm{\mu}\trans \bm{B}_n(\lambda) \bm{\mu}-U_{1}(\lambda;H_n,G_n,y_n)\right|+\left|\frac{2}{\sqrt{n_2}} \bm{\mu}\trans \bm{B}_n(\lambda) \bm{w}_2\right|\\&+\left|\frac{1}{n_1} \bm{w}_1\trans \bm{B}_n(\lambda) \bm{w}_1-y_{1n}T_{1}(\lambda;H_n,y_n)\right|+\left|\frac{1}{n_2} \bm{w}_2\trans \bm{B}_n(\lambda) \bm{w}_2-y_{2n}T_{1}(\lambda;H_n,y_n)\right|\\
			\leq&\frac{C}{n^{\left(1-\varepsilon\right)/2}},\label{14}
		\end{aligned}
	\end{equation}
	and similarly,
	\begin{align}
		&\left|A_{2 n}-U_{1}(\lambda;H_n,G_n,y_n)-\left(y_{1n}-y_{2n}\right)T_{1}(\lambda;H_n,y_n)\right| \leq \frac{C}{n^{\left(1-\varepsilon\right)/2}}, \label{15}\\
		&\left\lvert\, A_{3 n}-U_{2}(\lambda;H_n,G_n,y_n)-\left(y_{1n}+y_{2n}\right) T_{2}(\lambda;H_n,y_n) \right\lvert \leq \frac{C}{n^{\left(1-\varepsilon\right)/2}}. \label{16}
	\end{align}
	It is easy to verify that, under our assumptions, $T_{1}(\lambda;H_n,y_n),T_{2}(\lambda;H_n,y_n),U_{1}(\lambda;H_n,G_n,y_n)$ and $U_{2}(\lambda;H_n,G_n,y_n)$ are all bounded. Combining with \eqref{14}, \eqref{15} and \eqref{16}, we have
	\begin{align*}
		&\left|R_{RLDA}(\lambda)-\frac{1}{2}\sum_{i=1}^2\Phi\left(-\frac{U_{1}(\lambda;H_n,G_n,y_n)+(-1)^i\left(y_{1n}-y_{2n}\right)T_{1}(\lambda;H_n,y_n)}{2 \sqrt{U_{2}(\lambda;H_n,G_n,y_n)+\left(y_{1n}+y_{2n}\right) T_{2}(\lambda;H_n,y_n)}}\right)\right| \\
		\leq&\frac{1}{2}\sum_{i=1}^2\left|\Phi\left(-\frac{A_{in}}{2\sqrt{A_{3n}}}\right)-\Phi\left(-\frac{U_{1}(\lambda;H_n,G_n,y_n)+(-1)^i\left(y_{1n}-y_{2n}\right)T_{1}(\lambda;H_n,y_n)}{2 \sqrt{U_{2}(\lambda;H_n,G_n,y_n)+\left(y_{1n}+y_{2n}\right) T_{2}(\lambda;H_n,y_n)}}\right)\right|\\
		\leq&\frac{1}{2\sqrt{2\pi}}\sum_{i=1}^2\left|\frac{A_{in}}{2\sqrt{A_{3n}}}-\frac{U_{1}(\lambda;H_n,G_n,y_n)+(-1)^i\left(y_{1n}-y_{2n}\right)T_{1}(\lambda;H_n,y_n)}{2 \sqrt{U_{2}(\lambda;H_n,G_n,y_n)+\left(y_{1n}+y_{2n}\right) T_{2}(\lambda;H_n,y_n)}}\right|\\
		\leq &\frac{C}{n^{\left(1-\varepsilon\right)/2}}.
	\end{align*}
	The theorem is proven.
	\subsection*{Proof of Lemma \ref{lem2.4}}
	\begin{definition}[Stieltjes transform]
		For any distribution $G$ supported on $(0,\infty)$, we define its Stieltjes transform as
		\begin{align*}
			m_G(z):=\int\frac{1}{s-z}dG(s),\quad z\in\mathbb{C}^+
		\end{align*}
	\end{definition}
	\begin{definition}[companion Stieltjes transform]
		Recall that $\bm{S}_n$ is rewritten as $\frac{1}{n}\bm{\Sigma}^{\frac{1}{2}}\bm{X}\bm{X}\trans\bm{\Sigma}^{\frac{1}{2}}$, we define $\underline{m}(z)$ to be the Stieltjes transform for the limiting spectral distribution of $\frac{1}{n}\bm{X}\trans\bm{\Sigma} \bm{X}$, called companion
		Stieltjes transform.
	\end{definition}
	\begin{lemma}\label{lem4.1}
		Under the conditions of Lemma \ref{lem2.4}, for any $j\in\mathbb{J}$, deterministic unit vectors $\bm{\xi}\in\mathbb{R}^p$ and $z\in\mathbb{C}^{+}$, we have
		\begin{align*}
			\left|\bm{\xi}\trans\left(\bm{S}_n-z\bm{I}_p\right)^{-1}\bm{\xi}-\bm{\xi}\trans\left[-z\underline{m}(z)\bm{\Sigma}-z\bm{I}_p\right]^{-1}\bm{\xi}\right|\xrightarrow{a.s.} 0,
		\end{align*}
	\end{lemma}
	\begin{proof}
		See Theorem 1 in \cite{bai2007Asymptotics}.
	\end{proof}
	
	Define
	\begin{align*}
		\mathcal{C}_j=\{z\in\mathbb{C}:\widehat{\sigma}_{1j}\leq\Re(z)\leq\widehat{\sigma}_{2j},|\Im(z)|\leq c_0\},\quad j=1,\dots,p,
	\end{align*}
	where $c_0>0$ and $\widehat{\sigma}_{1j},\widehat{\sigma}_{2j}$ are chosen so that $\partial \mathcal{C}_j^-$ only encloses $a_j$ and excludes all other sample eigenvalues, and $\partial \mathcal{C}_j^-$ represents the negatively oriented boundary of $\mathcal{C}_j$. The existence of $\mathcal{C}_j$ is guaranteed by the Assumptions \ref{as33}. By the \emph{Cauchy integral}, we have the following equality
	\begin{align*}
		\bm{\xi}\trans \bm{u}_j\bm{u}_j\trans \bm{\xi}=\frac{1}{2\pi i}\oint_{\partial\mathcal{C}_j^-}\bm{\xi}\trans\left(\bm{S}_n-z\bm{I}_p\right)^{-1}\bm{\xi} dz.
	\end{align*}
	\begin{lemma}
		Under the conditions of Lemma \ref{lem2.4}, there is
		\begin{align*}
			\left|\frac{1}{2\pi i}\oint_{\partial\mathcal{C}_j^-}\bm{\xi}\trans\left(\bm{S}_n-z\bm{I}_p\right)^{-1}\bm{\xi} dz-\frac{1}{2\pi i}\oint_{\partial\mathcal{C}_j^-}\bm{\xi}\trans\left[-z\underline{m}(z)\bm{\Sigma}-z\bm{I}_p\right]^{-1}\bm{\xi} dz\right|\xrightarrow{a.s.}0.
		\end{align*}
	\end{lemma}
	\begin{proof}
		Define
		\begin{align*}
			\mathcal{C}_{1j}=\{z\in\mathbb{C}:\widehat{\sigma}_{1j}\leq\Re(z)\leq\widehat{\sigma}_{2j},|\Im(z)|= c_0\}
		\end{align*}
		and
		\begin{align*}
			\mathcal{C}_{2j}=\{z\in\mathbb{C}:\Re(z)\in\{\widehat{\sigma}_{1j},\widehat{\sigma}_{2j}\},|\Im(z)|\leq c_0\}.
		\end{align*}
		Then, the integral can be written in the following form
		\begin{equation}\label{9.55}
			\begin{aligned}
				&\left|\frac{1}{2\pi i}\oint_{\partial\mathcal{C}_j^-}\bm{\xi}\trans\left(\bm{S}_n-z\bm{I}_p\right)^{-1}\bm{\xi} dz-\frac{1}{2\pi i}\oint_{\partial\mathcal{C}_j^-}\bm{\xi}\trans\left[-z\underline{m}(z)\bm{\Sigma}-z\bm{I}_p\right]^{-1}\bm{\xi} dz\right|\\
				=&\frac{1}{2\pi i}\bigg|\oint_{\partial\mathcal{C}_{1j}^-}\bm{\xi}\trans\left(\bm{S}_n-z\bm{I}_p\right)^{-1}\bm{\xi} dz-\oint_{\partial\mathcal{C}_{1j}^-}\bm{\xi}\trans\left[-z\underline{m}(z)\bm{\Sigma}-z\bm{I}_p\right]^{-1}\bm{\xi} dz\\
				&+\oint_{\partial\mathcal{C}_{2j}^-}\bm{\xi}\trans\left(\bm{S}_n-z\bm{I}_p\right)^{-1}\bm{\xi} dz-\oint_{\partial\mathcal{C}_{2j}^-}\bm{\xi}\trans\left[-z\underline{m}(z)\bm{\Sigma}-z\bm{I}_p\right]^{-1}\bm{\xi} dz\bigg|\\
				\leq&\frac{1}{2\pi i}\left|\oint_{\partial\mathcal{C}_{1j}^-}\bm{\xi}\trans\left(\bm{S}_n-z\bm{I}_p\right)^{-1}\bm{\xi} dz-\oint_{\partial\mathcal{C}_{1j}^-}\bm{\xi}\trans\left[-z\underline{m}(z)\bm{\Sigma}-z\bm{I}_p\right]^{-1}\bm{\xi} dz\right|\\
				&+\frac{1}{2\pi i}\left|\oint_{\partial\mathcal{C}_{2j}^-}\bm{\xi}\trans\left(\bm{S}_n-z\bm{I}_p\right)^{-1}\bm{\xi} dz-\oint_{\partial\mathcal{C}_{2j}^-}\bm{\xi}\trans\left[-z\underline{m}(z)\bm{\Sigma}-z\bm{I}_p\right]^{-1}\bm{\xi} dz\right|.
			\end{aligned}
		\end{equation}
		
		For the first part, since $\left\|(\bm{S}_n-z\bm{I})^{-1}\right\|\leq 1/\Im(z)$ holds almost surely, by applying the \emph{Dominated convergence theorem} and Lemma \ref{lem4.1}, we obtain:
		\begin{equation}\label{9.56}
			\begin{aligned}
				&\left|\oint_{\partial\mathcal{C}_{1j}^-}\bm{\xi}\trans\left(\bm{S}_n-z\bm{I}_p\right)^{-1}\bm{\xi} dz-\oint_{\partial\mathcal{C}_{1j}^-}\bm{\xi}\trans\left[-z\underline{m}(z)\bm{\Sigma}-z\bm{I}_p\right]^{-1}\bm{\xi} dz\right|\\
				\leq&\oint_{\partial\mathcal{C}_{1j}^-}\left|\bm{\xi}\trans\left(\bm{S}_n-z\bm{I}_p\right)^{-1}\bm{\xi} -\bm{\xi}\trans\left[-z\underline{m}(z)\bm{\Sigma}-z\bm{I}_p\right]^{-1}\bm{\xi}\right| \left|dz\right|\xrightarrow{a.s.}0.
			\end{aligned}
		\end{equation}
		
		The proof of the second part is in the same spirit as that of Lemma 4 in \cite{liu2025Asymptotic}. Define an event $\Omega=\{\widehat{\sigma}_{1j}+c_1<a_j<\widehat{\sigma}_{2j}-c_1\}$, which holds almost surely for some small positive $c_1$ (independent of $n$). Then, $\left\|(\bm{S}_n-z\bm{I})^{-1}\right\|\leq 1/c_1$ holds almost surely. We have
		\begin{equation}
			\begin{aligned}\label{9.57}
				&\left|\oint_{\partial\mathcal{C}_{2j}^-}\bm{\xi}\trans\left(\bm{S}_n-z\bm{I}_p\right)^{-1}\bm{\xi} dz-\oint_{\partial\mathcal{C}_{2j}^-}\bm{\xi}\trans\left[-z\underline{m}(z)\bm{\Sigma}-z\bm{I}_p\right]^{-1}\bm{\xi} dz\right|\\
				=&\left|\oint_{\partial\mathcal{C}_{2j}^-\setminus\Re}\bm{\xi}\trans\left(\bm{S}_n-z\bm{I}_p\right)^{-1}\bm{\xi} dz-\oint_{\partial\mathcal{C}_{2j}^-\setminus\Re}\bm{\xi}\trans\left[-z\underline{m}(z)\bm{\Sigma}-z\bm{I}_p\right]^{-1}\bm{\xi} dz\right|\\
				\leq&\oint_{\partial\mathcal{C}_{2j}^-\setminus\Re}\left|\bm{\xi}\trans\left(\bm{S}_n-z\bm{I}_p\right)^{-1}\bm{\xi} -\bm{\xi}\trans\left[-z\underline{m}(z)\bm{\Sigma}-z\bm{I}_p\right]^{-1}\bm{\xi}\right| \left|dz\right|\xrightarrow{a.s.} 0.
			\end{aligned}
		\end{equation}
		Combining \eqref{9.55}, \eqref{9.56} and \eqref{9.57}, the lemma is proven.
	\end{proof}
	The above lemma simplifies the proof to calculating the following deterministic integral
	\begin{align*}
		\frac{1}{2\pi i}\oint_{\partial\mathcal{C}_j^-}\bm{\xi}\trans\left[-z\underline{m}(z)\bm{\Sigma}-z\bm{I}_p\right]^{-1}\bm{\xi} dz.
	\end{align*}
	Let $\omega(z)=-\frac{1}{\underline{m}(z)}$, we can write
	\begin{align*}
		&\frac{1}{2\pi i}\oint_{\partial\mathcal{C}_j^-}\bm{\xi}\trans \left[-z\underline{m}(z)\bm{\Sigma}-z\bm{I}_p\right]^{-1}\bm{\xi} dz\\
		=&\frac{1}{2\pi i}\oint_{\partial\Gamma_j^-}\bm{\xi}\trans\left(\frac{z}{\omega}-z\bm{I}_p\right)^{-1}\left(1-\frac{1}{n}\sum_{k=1}^{p}\frac{s_k^2}{\left(s_k-\omega\right)^2}\right)\bm{\xi} d\omega\\
		=&\frac{1}{2\pi i}\sum_{i=1}^{p}\oint_{\partial\Gamma_j^-}\frac{1}{s_i-\omega}\cdot\frac{1-\frac{1}{n}\sum_{k=1}^{p}\frac{s_k^2}{\left(s_k-\omega\right)^2}}{1-\frac{1}{n}\sum_{k=1}^{p}\frac{s_k}{s_k-\omega}}d\omega\bm{\xi}\trans v_iv_i\trans\bm{\xi},
	\end{align*}
	where $\partial\Gamma_j^-$ is a negatively oriented contour described by the boundary of the rectangle 
	\begin{align*}
		\Gamma_j=\left\{\omega\in\mathbb{C}:\sigma_{1j}\leq\Re(\omega)\leq\sigma_{2j},\left|\operatorname{Im}(\omega)\right|\leq c_0 \right\},
	\end{align*}
	which includes $s_j$ and excludes all the other population eigenvalues of $\bm{\Sigma}$.
	
	To solve this integral, we can use the \emph{Residue theorem}. Indeed, the function $\frac{1}{s_i-\omega}\cdot\frac{1-\frac{1}{n}\sum_{k=1}^{p}\frac{s_k^2}{(s_k-\omega)^2}}{1-\frac{1}{n}\sum_{k=1}^{p}\frac{s_k}{s_k-\omega}}$ is holomorphic on $\Gamma_j$, with the exception of two poles. The first pole is located at the eigenvalue $s_j$, by a calculation, the residue at $\omega=s_j$ can be expressed as follows:
	\begin{align}\label{w1}
		\operatorname{Res}\left(\frac{1}{s_i-\omega}\cdot\frac{1-\frac{1}{n}\sum_{k=1}^{p}\frac{s_k^2}{(s_k-\omega)^2}}{1-\frac{1}{n}\sum_{k=1}^{p}\frac{s_k}{s_k-\omega}},s_j\right)=\left\{ \begin{array}{ll}
			-n\left(1-\frac{1}{n}\sum_{k\neq j}^{p}\frac{s_k}{s_k-s_j}\right),  &j=i \\
			\frac{s_j}{s_j-s_i},  &j\neq i
		\end{array} \right.
	\end{align}
	The second pole $\omega_j$ is a solution to the equation \eqref{c7}. Similar to (37) in \cite{mestre2008Asymptotic}, the residues at $\omega=\omega_j$ can readily write
	\begin{align}\label{w2}
		\operatorname{Res}\left(\frac{1}{s_i-\omega}\cdot\frac{1-\frac{1}{n}\sum_{k=1}^{p}\frac{s_k^2}{(s_k-\omega)^2}}{1-\frac{1}{n}\sum_{k=1}^{p}\frac{s_k}{s_k-\omega}},\omega_j\right)=\frac{\omega_j}{s_i-\omega_j}.
	\end{align}
	Combining \eqref{w1} and \eqref{w2}, the proof is completed.
	\subsection*{Proof of Theorem \ref{thm2}}
	\begin{lemma}\label{lem4.3}
		Under the conditions of Theorem \ref{thm2}, we have
		\begin{align*}
			\left\|\sum_{j\in\mathbb{J}}\frac{\ell_j}{1-\ell_j}\left(\bm{u}_j\bm{u}_j\trans-\sum_{i=1}^{p}\chi_j(i)\bm{v}_j\bm{v}_j\trans\right)\right\|\xrightarrow{a.s.}0,\quad\left\|\bm{M}_n-\bm{W}_n\right\|\xrightarrow{a.s.}0,
		\end{align*}
		where $\bm{M}_n=\left[\bm{S}_n+\lambda\left(\bm{I}_p-\sum_{j\in\mathbb{J}}\ell_j\bm{u}_j\bm{u}_j\trans\right)\right]^{-1}$,   $\bm{W}_n=\bm{P}\left(\bm{P}\bm{S}_n\bm{P}+\lambda \bm{I}_p\right)^{-1}\bm{P}$ and $\bm{P}=\big(\bm{I}_p+\sum_{j\in\mathbb{J}}\\\frac{\ell_j}{1-\ell_j}\sum_{i=1}^{p}\chi_j(i)\bm{v}_i\bm{v}_i\trans\big)^{\frac{1}{2}}$. 
	\end{lemma}
	\begin{proof}
		The first conclusion can be directly obtained from Lemma \ref{lem2.4}. For the second conclusion, noting $\lambda_{min}\left(\bm{I}_p-\sum_{j\in\mathbb{J}}\ell_j\bm{u}_j\bm{u}_j\trans\right)\geq c$, we have $\|\bm{M}_n\|\leq 1/\lambda\cdot\lambda_{min}\left(\bm{I}_p-\sum_{j\in\mathbb{J}}\ell_j\bm{u}_j\bm{u}_j\trans\right)\\\leq c_1$. For $\bm{W}_n$, since $\|\bm{P}\|\leq c$, we have $\|\bm{W}_n\|\leq\|\bm{P}\|^2/\lambda\leq c_2$. Since
		\begin{align*}
			&\bm{M}_n-\bm{W}_n\\=&\left[\bm{S}_n+\lambda\left(\bm{I}_p-\sum_{j\in\mathbb{J}}\ell_j\bm{u}_j\bm{u}_j\trans\right)\right]^{-1}-\left[\bm{S}_n+\lambda \bm{P}^{-2}\right]^{-1}\\=&\lambda \bm{M}_n\left[\left(\bm{I}_p+\sum_{j\in\mathbb{J}}\frac{\ell_j}{1-\ell_j}\sum_{i=1}^{p}\chi_j(i)\bm{v}_i\bm{v}_i\trans\right)^{-1}-\left(\bm{I}_p+\sum_{j\in\mathbb{J}}\frac{\ell_j}{1-\ell_j}\bm{u}_j\bm{u}_j\trans\right)^{-1}\right]\bm{W}_n\\=&\lambda \bm{M}_n\left(\bm{I}_p+\sum_{j\in\mathbb{J}}\frac{\ell_j}{1-\ell_j}\sum_{i=1}^{p}\chi_j(i)\bm{v}_i\bm{v}_i\trans\right)^{-1}\\\cdot&\left[\sum_{j\in\mathbb{J}}\frac{\ell_j}{1-\ell_j}\left(\bm{u}_j\bm{u}_j\trans-\sum_{i=1}^{p}\chi_j(i)\bm{v}_i\bm{v}_i\trans\right)\right]\left(\bm{I}_p+\sum_{j\in\mathbb{J}}\frac{\ell_j}{1-\ell_j}\bm{u}_j\bm{u}_j\trans\right)^{-1}\bm{W}_n,
		\end{align*}
		thus we can show $\|\bm{M}_n-\bm{W}_n\|\leq \lambda \|\bm{M}_n\|\|\bm{W}_n\|\|\sum_{j\in\mathbb{J}}\frac{\ell_j}{1-\ell_j}(\bm{u}_j\bm{u}_j\trans-\sum_{i=1}^{p}\chi_j(i)\bm{v}_i\bm{v}_i\trans)\|\xrightarrow{a.s.}0$. The proof is completed.
	\end{proof}
	Recall
	\begin{align*}
		&\left(2 \bm{\mu}_1-\bar{\bm{x}}_1-\bar{\bm{x}}_2\right)\trans \left[\bm{S}_n+\lambda \left(\bm{I}_p-\sum_{j\in\mathbb{J}}\ell_j\bm{u}_j\bm{u}_j\trans\right)\right]^{-1}\left(\bar{\bm{x}}_1-\bar{\bm{x}}_2\right)\\=&\bm{\mu}\trans\bm{\Sigma}^{\frac{1}{2}} \bm{M}_n\bm{\Sigma}^{\frac{1}{2}}\bm{\mu}-\frac{1}{n_1}\bm{w}_1\trans\bm{\Sigma}^{\frac{1}{2}}\bm{M}_n\bm{\Sigma}^{\frac{1}{2}}\bm{w}_1+\frac{1}{n_2}\bm{w}_2\trans\bm{\Sigma}^{\frac{1}{2}}\bm{M}_n\bm{\Sigma}^{\frac{1}{2}}\bm{w}_2-\frac{2}{\sqrt{n_2}}\bm{w}_2\trans\bm{\Sigma}^{\frac{1}{2}}\bm{M}_n\bm{\Sigma}^{\frac{1}{2}}\bm{\mu}.
	\end{align*}
	For each part, it is trivial to show
	\begin{align*}
		\left|\bm{\mu}\trans\bm{\Sigma}^{\frac{1}{2}} \bm{M}_n\bm{\Sigma}^{\frac{1}{2}}\bm{\mu}-\bm{\mu}\trans\bm{\Sigma}^{\frac{1}{2}} \bm{W}_n\bm{\Sigma}^{\frac{1}{2}}\bm{\mu}\right|&\leq\|\bm{\mu}\|^2\|\bm{\Sigma}\|\|\bm{M}_n-\bm{W}_n\|\xrightarrow{a.s.}0,\\
		\left|\frac{1}{n_1}\bm{w}_1\trans\bm{\Sigma}^{\frac{1}{2}}\bm{M}_n\bm{\Sigma}^{\frac{1}{2}}\bm{w}_1-\frac{1}{n_1}\bm{w}_1\trans\bm{\Sigma}^{\frac{1}{2}}\bm{W}_n\bm{\Sigma}^{\frac{1}{2}}\bm{w}_1\right|&\leq\frac{\|\bm{w}_1\|^2}{n_1}\|\bm{\Sigma}\|\|\bm{M}_n-\bm{W}_n\|\xrightarrow{a.s.}0,\\
		\left|\frac{1}{n_2}\bm{w}_2\trans\bm{\Sigma}^{\frac{1}{2}}\bm{M}_n\bm{\Sigma}^{\frac{1}{2}}\bm{w}_2-\frac{1}{n_2}\bm{w}_2\trans\bm{\Sigma}^{\frac{1}{2}}\bm{W}_n\bm{\Sigma}^{\frac{1}{2}}\bm{w}_2\right|&\leq\frac{\|\bm{w}_2\|^2}{n_2}\|\bm{\Sigma}\|\|\bm{M}_n-\bm{W}_n\|\xrightarrow{a.s.}0,\\
		\left|\frac{2}{\sqrt{n_2}}\bm{w}_2\trans\bm{\Sigma}^{\frac{1}{2}}\bm{M}_n\bm{\Sigma}^{\frac{1}{2}}\bm{\mu}-\frac{2}{\sqrt{n_2}}\bm{w}_2\trans\bm{\Sigma}^{\frac{1}{2}}\bm{W}_n\bm{\Sigma}^{\frac{1}{2}}\bm{\mu}\right|&\leq\frac{2\|\bm{w}_2\|}{\sqrt{n_2}}\|\bm{\mu}\|\|\bm{\Sigma}\|\|\bm{M}_n-\bm{W}_n\|\xrightarrow{a.s.}0.
	\end{align*}
	Thus, we have
	\begin{align*}
		\left(2 \bm{\mu}_1-\bar{\bm{x}}_1-\bar{\bm{x}}_2\right)\trans \bm{M}_n\left(\bar{\bm{x}}_1-\bar{\bm{x}}_2\right)-\left(2 \bm{\mu}_1-\bar{\bm{x}}_1-\bar{\bm{x}}_2\right)\trans \bm{W}_n\left(\bar{\bm{x}}_1-\bar{\bm{x}}_2\right)\xrightarrow{a.s.}0,
	\end{align*}
	and similarly 
	\begin{align*}
		\left(2 \bm{\mu}_2-\bar{\bm{x}}_1-\bar{\bm{x}}_2\right)\trans \bm{M}_n\left(\bar{\bm{x}}_1-\bar{\bm{x}}_2\right)-\left(2 \bm{\mu}_2-\bar{\bm{x}}_1-\bar{\bm{x}}_2\right)\trans \bm{W}_n\left(\bar{\bm{x}}_1-\bar{\bm{x}}_2\right)\xrightarrow{a.s.}0.
	\end{align*}
	For the denominator, noting
	\begin{align*}
		\|\bm{M}_n\bm{\Sigma} \bm{M}_n-\bm{W}_n\bm{\Sigma} \bm{W}_n\|\leq&\|\bm{M}_n\bm{\Sigma} \bm{M}_n-\bm{M}_n\bm{\Sigma} \bm{W}_n\|+\|\bm{M}_n\bm{\Sigma} \bm{W}_n-\bm{W}_n\bm{\Sigma} \bm{W}_n\|\\\leq&(\|\bm{M}_n\|+\|\bm{W}_n\|)\|\bm{\Sigma}\|\|\bm{M}_n-\bm{W}_n\|\xrightarrow{a.s.}0,
	\end{align*}
	we can show
	\begin{align*}
		\left(\bar{\bm{x}}_1-\bar{\bm{x}}_2\right)\trans \bm{M}_n\bm{\Sigma} \bm{M}_n\left(\bar{\bm{x}}_1-\bar{\bm{x}}_2\right)-\left(\bar{\bm{x}}_1-\bar{\bm{x}}_2\right)\trans \bm{W}_n\bm{\Sigma} \bm{W}_n\left(\bar{\bm{x}}_1-\bar{\bm{x}}_2\right)\xrightarrow{a.s.}0.
	\end{align*}
	We simplify the study of the asymptotic performance of SEDA to the case of Corollary \ref{cor}. The proof is completed.
	\subsection*{Proof of some consistent estimates}
	In this subsection, we provide proofs for some consistent estimates proposed in this paper, including \eqref{alpha}, \eqref{T2}, \eqref{U1} and \eqref{U2}.
	\begin{lemma}\label{lem9.4.1}
		Under the conditions of Theorem \ref{thm2}, we have
		\begin{align}\label{38}
			\widehat{m}\xrightarrow{a.s.}m(-\lambda;H_f,y),\quad\widehat{m}'\xrightarrow{a.s.}m'(-\lambda;H_f,y),
		\end{align}
		where $m'$ is the derivative of $m$.
	\end{lemma}
	\begin{proof}
		By using Lemma \ref{lem4.3}, it can be shown that
		\begin{align*}
			&\left|\widehat{m}-\frac{1}{p}\tr\left(\bm{P}\bm{S}_n\bm{P}+\lambda \bm{I}_p\right)^{-1}\right|\\\leq&\left\|\left(\bm{S}_n\bm{\bm{\mathcal{I}}}^{-1}+\lambda\bm{I}_p\right)^{-1}-\left(\bm{P}\bm{S}_n\bm{P}+\lambda\bm{I}_p\right)^{-1}\right\|\\=&\left\|\left(\bm{S}_n\bm{\bm{\mathcal{I}}}^{-1}+\lambda\bm{I}_p\right)^{-1}\left(\bm{P}\bm{S}_n\bm{P}-\bm{S}_n\bm{\bm{\mathcal{I}}}^{-1}\right)\left(\bm{P}\bm{S}_n\bm{P}+\lambda\bm{I}_p\right)^{-1}\right\|\\\leq&\left\|\left(\bm{S}_n\bm{\bm{\mathcal{I}}}^{-1}+\lambda\bm{I}_p\right)^{-1}\right\|\left\|\left(\bm{P}\bm{S}_n\bm{P}-\bm{S}_n\bm{\bm{\mathcal{I}}}^{-1}\right)\right\|\left\|\left(\bm{P}\bm{S}_n\bm{P}+\lambda\bm{I}_p\right)^{-1}\right\|\\\leq&\frac{1}{\lambda^2}\left\|\bm{S}_n\right\|\left\|\bm{P}^2-\bm{\bm{\mathcal{I}}}^{-1}\right\|
			\xrightarrow{a.s.}0,
		\end{align*}
		and similarly
		\begin{align*}
			&\widehat{m}'-\frac{1}{p}\tr\left(\bm{P}\bm{S}_n\bm{P}+\lambda \bm{I}_p\right)^{-2}\xrightarrow{a.s.}0.
		\end{align*}
		Combining with the results in \cite{elkaroui2008Spectrum}
		\begin{align*}
			\frac{1}{p}\tr\left(\bm{P}\bm{S}_n\bm{P}+\lambda \bm{I}_p\right)^{-1}\xrightarrow{a.s.}m(-\lambda;H_f,y),
		\end{align*}
		and
		\begin{align*}
			\frac{1}{p}\tr\left(\bm{P}\bm{S}_n\bm{P}+\lambda \bm{I}_p\right)^{-2}\xrightarrow{a.s.}m'(-\lambda;H_f,y),
		\end{align*}
		we obtain that $\widehat{m}\xrightarrow{a.s.}m(-\lambda;H_f,y)$ and $\widehat{m}'\xrightarrow{a.s.}m'(-\lambda;H_f,y)$, the proof is completed.
	\end{proof}
	\begin{lemma}
		Under the conditions of Theorem \ref{thm2}, we have
		\begin{align*}
			\widehat{T}_1\xrightarrow{a.s.}T_1(\lambda;H_f,y),\quad \widehat{T}_2\xrightarrow{a.s.}T_2(\lambda;H_f,y),\quad \widehat{m}_1\xrightarrow{a.s.}m_1(-\lambda;H_f,y).
		\end{align*}
	\end{lemma}
	\begin{proof}
		According to the definitions, we have
		\begin{equation}\label{T1s}
			\begin{aligned}
				T_1(\lambda;H_f,y) &= \frac{1}{1-y+y\lambda m(-\lambda;H_f,y)}\int \left\{1-\frac{\lambda}{s\left[1-y+y\lambda m(-\lambda;H_f,y)\right]+\lambda}\right\}d H_f(s) \\
				&= \frac{1-\lambda m(-\lambda;H_f,y)}{1-y+y\lambda m(-\lambda;H_f,y)}.
			\end{aligned}
		\end{equation}
		Then, consider $T_2(\bm{\theta})$, by calculation, we can obtain,
		\begin{align}\label{30}
			m(-\lambda;H_f,y)&=\int\frac{s\left[1-y+y\lambda m(-\lambda;H_f,y)\right]+\lambda}{\left\{s\left[1-y+y\lambda m(-\lambda;H_f,y)\right]+\lambda\right\}^2}dH_f(s),
		\end{align}
		and
		\begin{align}\label{m'}
			m'(-\lambda;H_f,y)&=\int\frac{s\left[y m(-\lambda;H_f,y)-y\lambda m'(-\lambda;H_f,y)\right]+1}{\left\{s\left[1-y+y\lambda m(-\lambda;H_f,y)\right]+\lambda\right\}^2}dH_f(s).
		\end{align}
		Combining \eqref{30} and \eqref{m'}, we have
		\begin{align}\label{32}
			\int\frac{s}{\left\{s\left[1-y+y\lambda m(-\lambda;H_f,y)\right]+\lambda\right\}^2}dH_f(s)=\frac{m(-\lambda;H_f,y)-\lambda m'(-\lambda;H_f,y)}{1-y+y\lambda^2m'(-\lambda;H_f,y)},
		\end{align}
		and
		\begin{align}\label{33}
			&\int\frac{1}{\left\{s\left[1-y+y\lambda m(-\lambda;H_f,y)\right]+\lambda\right\}^2}dH_f(s)\\=&m'(-\lambda;H_f,y)-\frac{y\left[m(-\lambda;H_f,y)-\lambda m'(-\lambda;H_f,y)\right]^2}{1-y+y\lambda^2m'(-\lambda;H_f,y)}.
		\end{align}
		Substituting \eqref{T1s}, \eqref{32}, and \eqref{33} into the expression of $T_2(\bm{\theta})$, we can calculate to obtain
		\begin{align}\label{T2s}
			T_2(\lambda;H_f,y)=\frac{1-\lambda m(-\lambda;H_f,y)}{\left[1-y+y\lambda m(-\lambda;H_f,y)\right]^3}-\frac{\lambda m(-\lambda;H_f,y)-\lambda^2m'(-\lambda;H_f,y)}{\left[1-y+y\lambda m(-\lambda;H_f,y)\right]^4},
		\end{align}
		and
		\begin{align*}
			m_1(-\lambda;H_f,y)=\frac{1}{y\left[1-y+y\lambda m(-\lambda;H_f,y)\right]}-\frac{y\lambda\left(m(-\lambda;H_f,y)-\lambda m'(-\lambda;H_f,y)\right)}{y\left(1-y+y\lambda m(-\lambda;H_f,y)\right)^2}-\frac{1}{y}
		\end{align*}
		By Lemma \ref{lem9.4.1} and the \emph{Continuous mapping theorem}, the proof is completed.
	\end{proof}
	\begin{lemma}
		Under the conditions of Theorem \ref{thm2}, we have
		\begin{align*}
			\widehat{U}_1\xrightarrow{a.s.}U_1(\lambda;H_f,G_f,y),\quad \widehat{U}_2\xrightarrow{a.s.}U_2(\lambda;H_f,G_f,y).
		\end{align*}
	\end{lemma}	
	\begin{proof}
		We can directly deduce
		\begin{align*}
			\beta_j=\frac{\langle\bar{\bm{x}}_1-\bar{\bm{x}}_2,\bm{u}_j\rangle^2}{s_j\chi_j(j)}\xrightarrow{a.s.}\langle\bm{\mu},\bm{v}_j\rangle^2,\quad\widetilde{s}_j=s_j\left[1+\frac{\ell_j}{1-\ell_j}\chi_j(j)\right]\xrightarrow{a.s.}f(s_j),
		\end{align*}
		and
		\begin{align*}
			\gamma\xrightarrow{a.s.}&\sum_{j\in\mathbb{J}}\langle\bm{\mu},\bm{v}_j\rangle^2+\left(\|\bm{\mu}_1-\bm{\mu}_2\|^2-\sum_{j\in\mathbb{J}}\langle\bm{\mu}_1-\bm{\mu}_2,\bm{v}_j\rangle^2\right)/\sigma^2=\|\bm{\mu}\|^2.
		\end{align*}
		Combining with \eqref{38}, then we can complete the proof by the \emph{Continuous mapping theorem}.
	\end{proof}
	\begin{lemma}
		Under the conditions of Theorem \ref{thm2}, we have
		\begin{align*}
			\big|\widehat{\alpha}-\left(\alpha_0-\alpha_1\right)\big|\xrightarrow{a.s.}0.
		\end{align*}
	\end{lemma}
	\begin{proof}
		Combining Lemma \ref{lem2}, \ref{lem4} and \ref{lem9.4.1}, the lemma is proven.
	\end{proof}

	\newpage
	\bibliographystyle{apalike}
	\bibliography{ref}
	
\end{document}